\documentclass[10pt,a4paper]{article}

\usepackage{microtype}
\usepackage{graphicx}
\usepackage{subcaption}
\usepackage{booktabs} 
\usepackage{amsmath,amsthm,amsfonts,amssymb,verbatim}
\usepackage{tikz}
\usepackage{hyperref}
\usepackage{todonotes}
\usepackage{algorithm,algorithmicx}
\usepackage{algpseudocode}
\usepackage{diagbox}
\newtheorem{theorem}{Theorem}

\newtheorem{lemma}[theorem]{Lemma}
\newtheorem{corollary}[theorem]{Corollary}

\newtheorem{example}[theorem]{Example}

\def\R{\mathbb{R}}
\def\N{\mathbb{N}}

\DeclareMathOperator*{\argmax}{arg\,max}

\usepackage{authblk}

\begin{document}

\title{Ensemble Methods for Robust Support Vector Machines using Integer Programming}

\author[1]{Jannis Kurtz\footnote{j.kurtz@uva.nl}}

\affil[1]{Amsterdam Business School, University of Amsterdam, 1018 TV Amsterdam, Netherlands}

\date{}

\maketitle

\begin{abstract}
In this work we study binary classification problems where we assume that our training data is subject to uncertainty, i.e. the precise data points are not known. To tackle this issue in the field of robust machine learning the aim is to develop models which are robust against small perturbations in the training data. We study robust support vector machines (SVM) and extend the classical approach by an ensemble method which iteratively solves a non-robust SVM on different perturbations of the dataset, where the perturbations are derived by an adversarial problem. Afterwards for classification of an unknown data point we perform a majority vote of all calculated SVM solutions. We study three different variants for the adversarial problem, the exact problem, a relaxed variant and an efficient heuristic variant. While the exact and the relaxed variant can be modeled using integer programming formulations, the heuristic one can be implemented by an easy and efficient algorithm. All derived methods are tested on random and realistic datasets and the results indicate that the derived ensemble methods have a much more stable behaviour when changing the protection level compared to the classical robust SVM model.
\end{abstract}

\noindent\textbf{Keywords:} Robust Optimization, Support Vector Machines, Mixed-Integer Programming, Ensemble Methods

\section{Introduction}
Many practical applications can be modeled as binary classification problems, i.e. in a given data space we want to correctly assign one of two classes to each data point. Binary classification problems appear in various fields as text or document classification, speech recognition, fraud detection, medical diagnosis and many more; see \cite{ma2014support,wang2005support} for an overview. 

One popular approach to tackle binary classification problems are so called \textit{support vector machines} where, given a set of training data, the basic idea is to find a separating hyperplane which separates the data points of one class from the ones of the other class. Afterwards an unseen data point is classified by checking on which side of the hyperplane it is located. The first appearance of this approach dates back to the 1960s and was introduced in \cite{vapnik1964note} for the case of separable data. Later in the 1990s the approach was generalized to non-separable data in \cite{cortes1995support} and was studied intensively since then. The approach attained enormous popularity due to its simplicity and efficiency. Additionally it provides good theoretical generalization bounds which were derived by methods from statistical learning theory \cite{vapnik1998statistical,scholkopf1997support}. The SVM was later also generalized to multiclass classification problems; see \cite{hsu2002comparison}. For a derivation of the theoretical foundations of SVM see \cite{mohri2018foundations}.

To improve the accuracy of SVM models, ensemble methods were studied; see \cite{wang2009empirical} for an empirical comparison. In general the idea behind ensemble methods is to calculate a set of $k$ different classifiers and aggregate their decision by performing a majority vote. The two most popular ensemble methods are bagging and boosting \cite{kim2003constructing}. Bagging methods generate $k$ different training sets by randomly drawing samples from the original training set via a bootstrap technique \cite{breiman1996bagging,kim2002support}. On the other hand boosting methods iteratively generate new classifiers by defining new sample weights in each iteration, where the weights for data samples which are misclassified by the already determined classifiers are increased \cite{freund1996experiments}. One of the most famous boosting methods is AdaBoost \cite{freund1997decision,xing2020robust}.

While all the classical SVM models assume the training data to be precisely known, in practical applications the recorded data points can often be subject to uncertainty. This can be due to measurement or rounding errors or due to human errors. Furthermore in applications as spam mail detection or fraud detection an adversarial may be interested in changing its data point slightly such that a trained model is not able to classify it correctly \cite{kurakin2016adversarial}. Motivated by the field of \textit{robust optimization} (\cite{ben2009robust,gorissen2015practical,bertsimas2011theory,buchheim2018robust}) the classical SVM approach was extended to the robust setting where for each data point an uncertainty set is defined which contains all possible perturbations we want to be protected against. The task is then to find a separating hyperplane which separates the data points of one class and all of its perturbations from the data points of the other class and all its perturbations. This approach was already studied in several publications \cite{xu2009robustness,bertsimas2019robust,maggioni2021robustsvm,trafalis2007robust,trafalis2006robust, el2003robust}. In \cite{xu2009robustness} the robust model was studied for different classes of uncertainty sets and it was shown that adding robustness to the SVM model leads to implicit regularization. In \cite{bertsimas2019robust} the robust SVM model was analyzed for data uncertainty and label uncertainty. The authors in \cite{maggioni2021robustsvm} computationally compare the robust SVM model and the distributionally robust SVM model for several types of uncertainty sets. 

While adding robustness to machine learning models often leads to better generalization errors on perturbed data, the accuracy on clean data often decreases compared to the non-robust models. This \textit{trade-off between robustness and accuracy} was extensively studied in the ML literature \cite{zhang2019theoretically,raghunathan2020understanding,dobriban2020provable,tsipras2018robustness} and it was argued that robust models need larger training sets \cite{schmidt2018adversarially}. Furthermore the size of the uncertainty set which is used during the training process influences the mentioned trade-off and since we do not know in advance which size the future perturbations will have it is not well-defined what performance of a robust model is to be achieved. While increasing the size of the uncertainty set can lead to better accuracies on strongly perturbed data the accuracy can decrease on slightly perturbed or clean data. Hence finding the right size of the uncertainty sets is a difficult task which has to be solved by the user.

\paragraph{Contributions}
While there is comprehensive literature on ensemble methods for non-robust support vector machines, to the best of our knowledge ensemble methods were not used to improve the robustness of the SVM model regarding uncertainty in the data. In this work we present the following contributions:
\begin{itemize}
\item We derive a robust ensemble method which iteratively solves a non-robust SVM on different perturbations of the dataset, where the perturbations are derived by an adversarial problem. Afterwards we classify an unknown data point by performing a majority vote of all calculated SVM solutions. While in the classical robust SVM model we can only calculate one single hyperplane which has to protect us against all possible data perturbations, in the ensemble model the uncertainty is distributed over a set of hyperplanes. 
\item We study the exact adversarial problem and show that it can be modeled by an integer programming formulation for all classical uncertainty sets.
\item We propose a relaxed version of the adversarial problem where the average hinge-loss over all perturbation vectors is maximized. Using results from convex analysis we derive two integer programming reformulations for $\ell_1$ and $\ell_\infty$-uncertainty.
\item We propose an efficient heuristic variant of the adversarial problem, where the adversarial perturbation is calculated by a weighted mean of the hyperplane normal vectors.
\item We test all three methods on random and realistic datasets and compare the results to the classical robust SVM model and a non-robust ensemble SVM model based on bagging. 
\end{itemize}

The paper is organized as follows: in Section \ref{sec:preliminaries} we introduce the notation, the framework of classical SVM and the framework of classical robust SVM. In Section \ref{sec:MRO-SVM} we then introduce our robust ensemble method and afterwards study all three variants of the adversarial problem. Finally in Section \ref{sec:computations} we test all methods computationally and analyze the results which are concluded in Section \ref{sec:conclusion}.

\section{Preliminaries}\label{sec:preliminaries}

\subsection{Notation}
We define $[k]=\{1,\ldots ,k\}$ for all $k\in \N$ and the set of all non-negative real vectors is defined by $\R_+^n:=\left\{ x\in\R^n \ | \ x\ge 0\right\}$. For an arbitrary norm $\|\cdot\|$ in $\R^n$ we denote its dual norm by $\|\cdot \|^*$ which is defined as
\[
\|v\|^*= \sup_{w\in\R^n: \|w\|\le 1} v^\top w  \quad \forall v\in\R^n.
\]
The $\ell_p$-norm of $v\in\R^n$ where $p\in \N$ is denoted by $\|\cdot \|_p$ and defined as
\[
\| v\|_p:=\left( \sum_{i\in [n]} v_i^p \right)^\frac{1}{p}
\]
and the $\ell_\infty$-norm is defined as
\[
\| v\|_\infty:=\max_{j\in [n]} |v_j| .
\]
Furthermore we define the sign of a value $x\in \R$ by
\[
\text{sgn}(x):=\begin{cases} 1 & x\ge 0 \\ -1 & x<0 \end{cases} .
\]
Note that we make the unusual assumption that sgn$(0)=1$ since later we want to assign one of the two possible classes in $\{-1,1\}$ to each data point by using the sign-function and we therefore have to assign one of the two possibilities to the $0$-value. Nevertheless assigning $-1$ instead of $1$ is also possible. Finally we define $[x]_+=\max\{x,0\}$ for each $x\in\R$.

\subsection{Support Vector Machines}
Given a labeled training set $\mathcal D = \{(x^1,y^1), \ldots ,(x^m,y^m)\}$ with data points $x^j\in\R^n$ and labels $y^j\in \{-1,1\}$ for each $j\in [m]$, the idea of the classical support vector machine approach is to find a hyperplane $H:=\left\{ x\in\R^n: \ w^\top x + b=0\right\}$, where $w\in \R^n$ and $b\in \R$, such that the hyperplane separates the data set by its class labels. More precisely we want to find a hyperplane $H$ such that all data samples with label $y^j=1$ lie on one side of the hyperplane and all data samples with label $y_j=-1$ lie on the other side of the hyperplane in which case it must hold
\[
y^j(w^\top x^j + b)\ge 0 \quad \forall \ j\in[m].
\]
Clearly it is not always possible to separate the data by its class labels in which case we call the data set \textit{non-separable}. After we calculated an appropriate hyperplane, for each data point $x\in\R^n$ we assign the class $y=\text{sgn}(w^\top x + b)$. Since the sign function is discontinuous, instead of minimizing the true empirical error, we minimize the \textit{hinge-loss} which is defined as
\[
[1-y^j(w^\top x^j + b)]_+ \quad \forall \ j\in[m].
\]
Note that the hinge loss is a convex and continuous function over the weight-variables $w$ and $b$. The classical linear SVM approach is then to calculate the optimal hyperplane parameters $w^*,b^*$ of the convex problem
\begin{equation}\label{eq:svm}\tag{SVM}
\min_{w\in\R^n, b\in\R} \ \| w\|_2^2 + \sum_{j=1}^{m} c_j[1- y^j(w^\top x^j + b)]_+ 
\end{equation}
which is equivalent to the quadratic optimization problem
\begin{equation}
\begin{aligned}
\min_{w,b,\xi} \ & \| w\|_2^2 + \sum_{j=1}^{m}c_j \xi_j \\
s.t. \quad \ & \xi_j \ge 1- y^j(w^\top x^j + b)\quad \forall j\in [m] \\
& w\in\R^n, b\in\R, \xi\in \R_+^m .
\end{aligned}
\end{equation}
Here $c_j\in\R_+$ are fixed hyper-parameters which can be used to adjust the impact each data point has on the empirical loss. Often these parameters are all set to a constant $C\in\R_+$ (e.g. $C=1$) or are derived in a validation process. Note that the regularization term $\| w\|_2^2$ is used to maximize the margin between the hyperplane $H$ and the data. The points $x^j$ with $y^j(w^\top x^j + b) =1$ are called \textit{support vectors} (see \cite{mohri2018foundations} for detailed explanations).

\subsection{Robust Support Vector Machines}
In the robust setting we assume that each data point can be perturbed by a vector $\delta$ which is contained in a bounded region. More precisely for each training sample $x^j$ we define a radius $r_j\ge 0$ and an \textit{uncertainty set} \[U_j=\left\{ \delta\in \R^n \ | \ \|\delta \|\le r_j\right\}\] containing all possible perturbation vectors. A perturbed data point is then given by $x^j + \delta$ where $\delta\in U_j$. Here we can choose an arbitrary norm $\|\cdot \|$ to define the sets $U_j$. Often either the $\ell_1,\ell_2$ or $\ell_\infty$ norm is used (see \cite{maggioni2021robustsvm} for a computational comparison). We assume that a perturbed data point $x^j$ has the same label $y^j$ as the original data point. As for the classical SVM model the goal is to find a separating hyperplane $H_{\text{RO}}:=\left\{ x\in\R^n: \ w^\top x + b=0\right\}$ which in this case is robust against all possible data perturbations contained in the uncertainty sets, i.e. which predicts the true label of $x^j$ for all points in $x_j+U_j$. Such an optimal robust hyperplane is given by an optimal solution $w^*,b^*$ of the problem
\begin{equation}\tag{RO-SVM}\label{eq:ro-svm}
\min_{w\in\R^n, b\in\R} \ \sum_{j=1}^{m} \max_{\delta^j \in U_j} \ [1- y^j(w^\top (x^j+\delta^j) + b)]_+
\end{equation}
which can be reformulated by level-set transformation as
\begin{equation}
\begin{aligned}
\min_{w,b,\xi} \ & \sum_{j=1}^{m} \xi_j \\
s.t. \quad \ & \xi_j \ge 1- y^j(w^\top (x^j+\delta) + b) \quad \forall \delta\in U_j, \ j\in [m] \\
& w\in\R^n, b\in\R, \xi\in \R_+^m .
\end{aligned}
\end{equation}
Note that the latter problem has infinitely many constraints and hence cannot be used for practical computations in this form. However by using the classical reformulation used in robust optimization we can reformulate the problem as follows. For each $j\in [m]$ the constraints
\[
\xi_j \ge 1- y^j(w^\top (x^j+\delta) + b) \quad \forall \delta\in U_j
\]
can equivalently be reformulated as 
\[
\xi_j \ge \max_{\delta\in U_j} \ 1- y^j(w^\top (x^j+\delta) + b)
\]
which is equivalent to
\begin{equation}\label{eq:robustConstraint}
\xi_j \ge 1- y^j(w^\top x^j + b) + r_j\|w\|^*
\end{equation}
by applying the definition of the dual norm. Substituting the latter constraint for each $j\in [m]$ into the problem we obtain the reformulation
\begin{equation}\label{eq:robustSVMreformulation}
\begin{aligned}
\min_{w,b,\xi} \ & \sum_{j=1}^{m} \xi_j \\
s.t. \quad \ & \xi_j \ge 1- y^j(w^\top (x^j) + b) + r_j\|w\|^* \quad \forall \ j\in [m] \\
& w\in\R^n, b\in\R, \xi\in \R_+^m 
\end{aligned}
\end{equation}
which is now a problem with $m$ constraints. Note that depending on which norm $\|\cdot\|$ we choose to model the uncertainty sets the constraints can have different structures. Since the dual norm of the $\ell_2$-norm is the $\ell_2$-norm, in this case we obtain a quadratic problem. On the other hand for the $\ell_1$ and $\ell_\infty$-norm (which are the dual norms of each other) we can transform the latter problem into a linear problem by adding additional variables. Hence in all these cases Problem \eqref{eq:ro-svm} can be solved by state-of-the-art solvers like CPLEX or Gurobi \cite{cplex2009v12,gurobi}. Note that we do not add any regularization term to the objective function of Problem \eqref{eq:ro-svm} since it was shown in \cite{xu2009robustness} that adding robustness to the model induces regularization implicitly. This can also be seen in the Reformulation \eqref{eq:robustSVMreformulation}, since each constraint contains an implicit regularization term $r_j\|w\|^*$. 

One of the main difficulties in robust machine learning is finding the right size of the uncertainty set, i.e. the right defense radii $r_j$. While finding the right hyperparameters often can be done via a validation step, in robust machine learning the main problem is to define what is the ''right'' radius. One question is, if we should define different radii for different data points. This can be reasonable if the feature values of some data points have a different magnitude than the others or if this is a reasonable assumption for the specific application. A solution which is often used is to normalize the data and then choosing the same defense radius for each data point. Another problem is the following: As already mentioned in the introduction, adding robustness to a machine learning model normally leads to better accuracies on perturbed data since our model learned to be protected against perturbations. At the same time a reduction of its accuracy on clean data can be observed. The situation is even more complex since we can vary the size of the perturbations (also called attacks) which leads to different accuracy values for each attack-size. This development can be plotted in a trade-off curve (see Section \ref{sec:computations}). Since in practical applications we do not know in advance which size the attacks will have, it is not easy to decide which trade-off curve we desire. If we can assume that the attacks will be small we maybe want to have larger accuracies for small attacks while at the same time the accuracy for larger attacks is allowed to be worse. If we want to have a more robust model for large attacks we may desire the opposite behavior. One way to determine the right radius of our uncertainty sets is to calculate the average accuracy over all possible attack sizes and choose the radius which leads to the best average accuracy (see trade-off curves in Section \ref{sec:computations}). Nevertheless the best situation would be to find a robust model with a stable behaviour regarding the different radii. In the following section we derive an ensemble method which turns out to be much more stable for varying defense-levels than Problem \eqref{eq:ro-svm}, sometimes coming with a small reduction in accuracy.    

\section{Robust Ensemble Methods Based on Majority Votes}\label{sec:MRO-SVM}
In this section we derive an ensemble method to improve robustness and reduce the trade-off between robustness and accuracy of the classical robust SVM model. 

We assume the same setup as in Section \ref{sec:preliminaries}, i.e. we have given a labeled training set $\mathcal D = \{(x^1,y^1), \ldots ,(x^m,y^m)\}$ with data points $x^j\in\R^n$ and labels $y^j\in \{-1,1\}$ for each $j\in [m]$ and we assume that each training sample $x^j$ can be perturbed by any perturbation vector $\delta$ contained in the uncertainty set $U_j=\left\{ \delta\in \R^n \ | \ \|\delta \|\le r_j\right\}$ where $r_j\ge 0$ are given radii. The main drawback of the classical robust model \eqref{eq:ro-svm} is that we can only calculate a single hyperplane to hedge against all possible data perturbations for all training samples. Depending on the size of the uncertainty sets $U_j$ this can lead to bad solutions where for each sample $x^j$ a perturbation $\delta^j\in U_j$ can be found such that $x^j+\delta^j$ is misclassified; see Figure \ref{fig:exampleMRO-SVM}.

Motivated by the idea of min-max-min robustness (\cite{buchheim2017min,buchheim2018complexity,kurtz2021new}) the objective of the following ensemble model is to find $k$ hyperplanes \[H_1:=\left\{ x\in\R^n: \ w_1^\top x + b_1=0\right\},\ldots , H_k:=\left\{ x\in\R^n: \ w_k^\top x + b_k=0\right\}\] to hedge against all possible data perturbations where the parameter $k\in \N$ is fixed in advance. As in classical ensemble methods we afterwards classify each data point by \textit{majority vote} of all hyperplanes, i.e. for a given data point $x\in\R^n$ the predicted class is
\begin{equation}\label{eq:majorityvote}
y = \text{sgn}\left( \sum_{i=1}^{k} \text{sgn}(w_i^\top x + b_i)\right) .
\end{equation}
Note that the latter value is $0$ if exactly half of the hyperplanes assign class $1$ and half of it assign class $-1$. By our definition of the sign-function we predict the class $1$ in this case. However we could predict an arbitrary class, e.g. the class which is more often contained in the training set. The advantage of this ensemble approach is that the set of perturbations is now distributed over $k$ hyperplanes instead of a single one. Hence each of the $k$ hyperplanes may misclassify certain perturbed data samples, but this error is often canceled out by a majority of hyperplanes classifying this perturbed point correctly. In Figure \ref{fig:exampleMRO-SVM} we show a two-dimensional example where it is not possible to find a single hyperplane which classifies all perturbed data points correctly, while it is possible if we can choose $k=3$ hyperplanes and perform the majority vote \eqref{eq:majorityvote}.

\begin{figure}[htb]
\begin{minipage}{\linewidth}
	\begin{minipage}[t]{0.45\textwidth}
	\centering
\begin{tikzpicture}[scale=0.5]
\node (1) at (-2,0) [draw,circle,inner sep=1.5pt] {};
\node (2) at (0,2) [draw,circle,inner sep=1.5pt] {};
\node (3) at (0.5,-0.5) [draw,fill,circle,inner sep=1.5pt] {};
\draw (0.5,-0.5) -- (1.5,-0.5);
\node (e) at (1.05,-0.2) {\small $r_j$};
\draw (-3,-1) -- (-3,1) -- (-1,1) -- (-1,-1) -- (-3,-1);
\draw (-1,1) -- (-1,3) -- (1,3) -- (1,1) -- (-1,1);
\draw [thick] (-0.5,-1.5) -- (-0.5,0.5) -- (1.5,0.5) -- (1.5,-1.5) -- (-0.5,-1.5);
\draw [-, thick, red!100] (-3,-2.5) -- (2.5,3);
\end{tikzpicture}

	\end{minipage}
	 \
	\begin{minipage}[t]{0.45\textwidth}
\centering
\begin{tikzpicture}[scale=0.5]
\node (1) at (-2,0) [draw,circle,inner sep=1.5pt] {};
\node (2) at (0,2) [draw,circle,inner sep=1.5pt] {};
\node (3) at (0.5,-0.5) [draw,fill,circle,inner sep=1.5pt] {};
\draw (0.5,-0.5) -- (1.5,-0.5);
\node (e) at (1.05,-0.2) {\small $r_j$};
\draw (-3,-1) -- (-3,1) -- (-1,1) -- (-1,-1) -- (-3,-1);
\draw (-1,1) -- (-1,3) -- (1,3) -- (1,1) -- (-1,1);
\draw [thick] (-0.5,-1.5) -- (-0.5,0.5) -- (1.5,0.5) -- (1.5,-1.5) -- (-0.5,-1.5);
\draw [-, thick, red!100] (-0.65,-2.5) -- (-0.65,3.5);
\draw [-, thick, blue!100] (-4,0.65) -- (2,0.65);
\draw [-, thick, green!100] (-2.5,-2.5) -- (2.5,2.5);

\node (s1) at (-2.2,-0.5) {\scriptsize $2$};
\node (s2) at (-0.1,1.5) {\scriptsize $2$};
\node (s3) at (0.2,-1) {\scriptsize $3$};
\node (s4) at (-0.8,2.5) {\scriptsize $3$};
\node (s5) at (-2.2,0.8) {\scriptsize $3$};
\node (s6) at (-0.2,0.2) {\scriptsize $2$};

\end{tikzpicture}

\end{minipage}
\end{minipage}
\caption{Seperating hyperplane for $\ell_\infty$ perturbations (left) and three separating hyperplanes for $\ell_\infty$ perturbations with majority vote (right). The numbers denote the number of hyperplanes classifying all points in the corresponding region correctly.}
\label{fig:exampleMRO-SVM}
\end{figure}
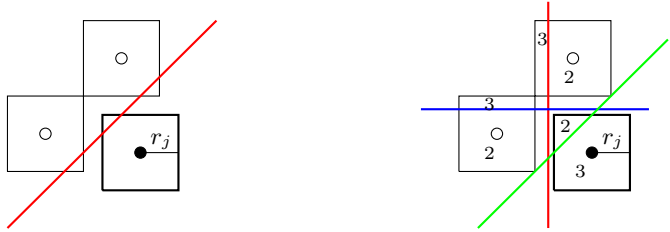

To achieve a robust set of $k$ hyperplanes as intended above, we propose the following algorithm which is similar to a boosting method. The basic idea of the algorithm is to run $k$ iterations and to calculate a new hyperplane in each iteration by solving the classical SVM problem \eqref{eq:svm} for a perturbed training set where each training sample is perturbed by an adversarial perturbation. To this end we denote the adversarial problem which returns and adversarial perturbation for given hyperplanes $H_1,\ldots, H_t$ and data point $x^j$ with label $y^j$ by $Adv(H_1,\ldots ,H_t, y^j, x^j)$. We will consider different adversarial problems later.

To describe the method more precisely, assume we already calculated the first $t$ hyperplanes $H_1,\ldots ,H_t$. Then in the next iteration we solve an adversarial problem $Adv(H_1,\ldots ,H_t, y^j, x^j)$ (to be specified later) for each $j\in [m]$ to obtain a corresponding worst-case perturbation $\delta^j$. Then we perturb each original data point by $\delta^j$, i.e. we calculate $\bar x^j = x^j + \delta^j$, and afterwards we apply the classical SVM to the perturbed labeled dataset $\bar{\mathcal{D}}=\{(\bar x^1,y^1),\ldots ,(\bar x^m,y^m)\}$ to obtain the next hyperplane $H_{t+1}=\{w_{t+1}^\top x + b_{t+1} = 0\}$. Since we want the SVM to focus on data points $\bar x^j$ which are not correctly classified by the already known hyperplanes, we define the following data sample weights which are passed to the SVM solver. For data sample $\bar x^j$ we define the corresponding weight $\lambda_j^t$ in iteration $t$ by $\lambda_j^t:=\frac{1}{1+\gamma_j^t}$ where
\begin{equation}\label{eq:weightsheuristic}
\gamma_j^t:= t + y^j\left( \sum_{i=1}^{t} \text{sgn}\left( w_i^\top \bar x^j + b_i\right)\right) .
\end{equation}
Note that $\gamma_j^t\in \{ 0,\ldots , 2t\}$ where $\gamma_j^t=0$ if all hyperplanes misclassify point $\bar x^j$ and $\gamma_j^t=2t$ if all hyperplanes correctly classify point $\bar x^j$. Particularly the more hyperplanes misclassify the point, the smaller the value $\gamma_j^t$, and the larger is the weight parameter $\lambda_j^t$. Hence using these weights the SVM focuses more on data points which are misclassified by a large number of hyperplanes. Note that in the first iteration we can calculate an arbitrary hyperplane, e.g. the classical robust SVM solution by solving Problem \eqref{eq:ro-svm}. The latter procedure is presented in Algorithm~\ref{alg:heuristic}. 

Note that for a data point $x^j$ we only need to consider the hyperplanes for which an adversarial perturbation $\delta\in U_j$ exists which leads to a misclassification of point $x^j+\delta$. We call such a hyperplane \textit{foolable}. A hyperplane $H_i$ is foolable if and only if 
\begin{equation}\label{eq:check_foolable}
\max_{\delta\in U_j} -y^j(w_i^\top (x^j+\delta)+b_i)>0
\end{equation}
since the latter condition says that we can find a perturbation $\delta\in U_j$ such that $y^j(w_i^\top (x^j+\delta)+b_i)<0$ and hence the point $x^j+\delta$ is misclassified by hyperplane $H_i$. Note that by definition of the dual norm it holds
\[
\max_{\delta\in U_j} -y^j(w_i^\top (x^j+\delta)+b_i) = -y^j(w_i^\top x^j+b_i) + r_j\|w_i\|^*
\]
and therefore we can solve Problem \eqref{eq:check_foolable} efficiently if we can calculate the dual norm efficiently. In the first step of the inner loop in Algorithm \ref{alg:heuristic} we calculate all foolable hyperplanes. Note that this step can be made more efficient by memorizing for all data points the hyperplanes which can be fooled. Then in the first step of the inner loop we only have to check if the last hyperplane can be fooled. 

\begin{algorithm}\caption{(Robust Ensemble Method)}\label{alg:heuristic}
\begin{algorithmic}
\State {\bfseries Input:} $n$,$m$, $\mathcal D$, $U_1,\ldots ,U_m$, $k$
	\State {\bfseries Output:} Hyperplane parameters $(w_1,b_1), \ldots, (w_k,b_k)$.
	\State Calculate an optimal solution $(w_1^*,b_1^*)$ of the classical robust SVM-Problem \eqref{eq:ro-svm}.
	
	\For{$t=2,\ldots ,k$ } 
	\For{$j=1,\ldots ,m$}
	\State Calculate all foolable hyperplanes $H_{i_1},\ldots ,H_{i_p}\in \{H_1,\ldots , H_t\}$ by
	\par\hskip\algorithmicindent solving \eqref{eq:check_foolable} for all $i=1,\ldots ,t$.\par
	\State Calculate an adversarial perturbation $\delta^j$ by solving problem \[Adv(H_{i_1},\ldots ,H_{i_p}, y^j, x^j)\]
	\State Set $\bar x^j:=x^j + \delta^j$.
	\State Calculate $\gamma_j^t$ as in \eqref{eq:weightsheuristic} for each $j\in [m]$.
	\State Set $\lambda_j^t:=\frac{1}{1+\gamma_j^t}$ for each $j\in [m]$.
	\State Calculate an optimal solution $(w_t^*,b_t^*)$ of \eqref{eq:svm} with sample weights 
	\par\hskip\algorithmicindent$c_j=\lambda_j^t$ and training set $\bar{\mathcal{D}}=\{(\bar x^1,y^1),\ldots ,(\bar x^m,y^m)\}$.\par
	\EndFor
	\EndFor
	\State \textbf{Return:} $(w_1^*,b_1^*), \ldots, (w_k^*,b_k^*)$
\end{algorithmic}
\end{algorithm}

\subsection{The Adversarial Problem}
One of the main components of Algorithm \ref{alg:heuristic} is the calculation of the adversarial perturbations. Clearly the best choice for the adversarial is to find a perturbation $\delta^j\in U_j$ such that the perturbed point $x^j+\delta^j$ is misclassified by as many hyperplanes as possible. We say that the adversary \textit{fools hyperplane $H_t$} with respect to data point $x^j$ if a perturbation $\delta^j\in U_j$ is returned with
\[
y^j(w_t^\top (x^j+\delta^j) + b_t)< 0, 
\]
i.e. if the perturbed point is misclassified by hyperplane $H_t$. We say that the adversary \textit{fools the ensemble model} with respect to data point $x^j$ if a perturbation $\delta^j\in U_j$ is returned such that at least $\lceil\frac{k+1}{2}\rceil$ hyperplanes are fooled with respect to $x^j$. Note that in the latter case more than half of the hyperplanes misclassify the perturbed point $x^j+\delta^j$ and hence the ensemble model misclassifies $x^j$.

We define the \textit{Exact Adversarial Problem} as follows: for given Hyperplanes $H_1,\ldots , H_t$ and a data point $x^j$ with label $y^j$, find a perturbation $\delta^j\in U_j$ such that the number of fooled hyperplanes is maximized. In the following lemma we show that this problem can be modeled as a mixed-integer program.

\begin{lemma}\label{lem:exactadversarialproblem}
Given hyperplanes $H_1=\{ x\in \R^n: w_1^\top x + b_1\} , \ldots , H_k=\{ x\in \R^n: w_k^\top x + b_k\}$, then the Exact Adversarial Problem for data point $x^j$ can be solved by solving
\begin{equation}\label{eq:adversarialProblem}
\begin{aligned}
\min \ & \sum_{i=1}^{k} z_i \\
s.t. \quad & y^j((w_i)^\top (x^j+\delta) + b_i) \le M z_i \quad \forall \ i\in [k] \\
& \|\delta\|\le r_j\\
& \delta\in \R^n, z\in \{0,1\}^k .
\end{aligned}
\end{equation}
where $M:=\|w_i\|_2(\|x^j\|_2+R)+|b_i|$ and $R=\max_{\delta\in U_j}\|\delta\|_2$. Furthermore the optimal value equals the number of hyperplanes which are not fooled.
\end{lemma}
\begin{proof}
First note that the constraint $\|\delta\|\le r_j$ ensures that $\delta\in U_j$. Now fix any hyperplane $i\in [k]$ and denote by $(\delta^*,z^*)$ an optimal solution of Problem \eqref{eq:adversarialProblem}. Clearly if $y^j((w_i)^\top (x^j+\delta^*) + b_i)> 0$ then $z_i=1$ must hold for each feasible solution and hence for the optimal solution. Furthermore in this case each $\delta\in U_j$ is feasible since by applying the Cauchy-Schwarz inequality and the triangle inequality we obtain 
\[
|y^j((w_i)^\top (x^j+\delta) + b_i)| \le \|w_i\|_2(\|x^j\|_2 + \| \delta\|_2) + |b_i| \le M.
\]
On the contrary, since we minimize the sum of all $z$-variables, if possible we want to set $z_i=0$. Hence if $y^j((w_i)^\top (x^j+\delta) + b_i)\le 0$ then $z_i^*=0$ must hold in the optimal solution. We can deduce that data point $x^j+\delta^*$ is misclassified by hyperplane $i$ if and only if $z_i^*=0$. Since we minimize the sum over all $z_i$, the optimal $\delta^*$ is the perturbation which maximizes the number of hyperplanes which are fooled. The last result follows from the argumentation above. 
\end{proof}
Note that since the optimal value $v^*$ of Problem \eqref{eq:adversarialProblem} is equal to the number of non-fooled hyperplanes, we can deduce that if $v^*\le \lfloor\frac{k-1}{2}\rfloor$, then the ensemble model is fooled by the adversary with respect to $x^j$ while if $v^*\ge \lceil\frac{k+1}{2}\rceil$ then no $\delta\in U_j$ exists which fools the ensemble model. If $k$ is an even number and $v^*=\frac{k}{2}$, i.e. exactly half of the hyperplanes are fooled, then it depends on the label we assign in this case to $x^j$ if the ensemble model is fooled or not.

Note that in theory if $k$ is a small value, we can solve the Exact Adversarial Problem by considering all $2^k$ possible subsets of hyperplane indices $S\subset [k]$ and for each consider the problem where we assume that all hyperplanes in $S$ classify data point $x^j$ correctly and all other hyperplanes misclassify $x^j$. Hence for each $S$ we only have to check if a feasible adversarial example exists i.e. we have to solve the linear continuous problem
\begin{equation*}
\begin{aligned}
\min \ & 0 \\
s.t. \quad & y^j((w_i)^\top (x^j+\delta) + b_i) > 0 \quad \forall \ i\in S \\
& y^j((w_i)^\top (x^j+\delta) + b_i) \le 0 \quad \forall \ i\in [k]\setminus S \\
& \delta\in U_j.
\end{aligned}
\end{equation*}
In practical computations the first set of strict inequalities can be replaces by inequalities $ y^j((w_i)^\top (x^j+\delta) + b_i) \ge \varepsilon$ where $\varepsilon$ is a small enough value. However this approach is more of theoretical interest since the number of problems we have to solve is exponential in $k$. 

Lemma \ref{lem:exactadversarialproblem} shows that the Exact Adversarial Problem can be modeled as a mixed-integer program with $k$ binary variables and $n$ continuous variables. The structure of the problem depends on the uncertainty set $U_j$. If we choose the $\ell_2$-norm to define the set, then we obtain a quadratic mixed-integer problem, while for the $\ell_1$ or the $\ell_\infty$-norm the problem can be reformulated as a linear mixed-integer problem by adding additional variables to model the absolute values appearing in the norm constraint. If $k$ is not too large all these problems can be solved by classical state-of-the-art solvers. Unfortunately Problem \eqref{eq:adversarialProblem} involves big-M constraints which are known to be computationally challenging. Since in Algorithm \ref{alg:heuristic} we have to solve $m$ adversarial problems in each iteration this can still lead to large computation times on realistic datasets. Hence a fast heuristic for finding possibly non-optimal adversarial perturbations is desired. We present such an heuristic in Section \ref{sec:heuristicadversarial}. Furthermore we consider a relaxed version of the Exact Adversarial Problem in Section \ref{sec:relaxedadversarial} where instead of the number of fooled hyperplanes the average hinge-loss is maximized.

\subsection{Relaxed Adversarial Problem}\label{sec:relaxedadversarial}
In this section we consider a relaxed version of the Exact Adversarial Problem. The idea is instead of maximizing the number of fooled hyperplanes, to maximize the average hinge-loss over all hyperplanes. We define this problem as follows: Given hyperplanes $H_1=\{ x\in \R^n: w_1^\top x + b_1\} , \ldots , H_k=\{ x\in \R^n: w_k^\top x + b_k\}$ the \textit{Relaxed Adversarial Problem} for data point $x^j$ is defined as
\begin{equation}\label{eq:adversarialProblemRelaxed}
\max_{\delta\in U_j} \ \sum_{i=1}^{k} [1- y^j((w_i)^\top (x^j+\delta) + b_i)]_+.
\end{equation}
The idea is that for a fooled hyperplane the hinge loss $[1- y^j((w_i)^\top (x^j+\delta) + b_i)]_+$ is larger than for a non-fooled hyperplane. Hence maximizing the average hinge-loss can lead to useful adversarial perturbations which can be used in Algorithm \ref{alg:heuristic}. Unfortunately since this problem involves maximizing a convex function, calculating an optimal solution is very challenging. In the following we will apply the results from \cite{selvi2020convex} to obtain useful reformulations of Problem \eqref{eq:adversarialProblemRelaxed} which can be solved by state-of-the-art integer programming solvers like CPLEX or Gurobi.

We consider the adversarial problem \eqref{eq:adversarialProblemRelaxed} for a fixed $j\in [m]$ and assume that we have given hyperplane parameters $w_1,\ldots ,w_k\in\R^n$, $b_1,\ldots ,b_k\in \R$. In the following we define the function $f:\R^k\to \R_+$ with \[f(t_1,\ldots ,t_k)=\sum_{i=1}^{k}[t_i]_+,\] the matrix $A\in \R^{k\times n}$ where the $i$-th row is given by $A_i=-y_j w_i^\top$ and the vector $c\in\R^k$ where $c_i=1-y_j(w_i^\top x^j +b_i)$. Then we can reformulate Problem \eqref{eq:adversarialProblemRelaxed} by
\begin{equation}\label{eq:adversarial_denhertog_variant}
\max_{\|\delta\|\le r_j} f(A\delta + c).
\end{equation}
The convex conjugate function $f^*$ of $f$ is given by
\[
f^*(v):=\sup_{t\in\R^k}v^\top t-f(t) = \begin{cases}
0 & v\in [0,1]^k \\
\infty & \text{else}
\end{cases}
\]
and hence the domain of $f^*$ is given by $[0,1]^k$. We can now apply the results in \cite{selvi2020convex} and obtain the following general result.
\begin{theorem}\label{thm:denHertogResult}
If $U_j=\left\{ \delta\in \R^n \ | \ \|\delta\|_p \le r_j\right\}$, then the optimal value of the Relaxed Adversarial Problem \eqref{eq:adversarialProblemRelaxed} is equal to the optimal value of
\begin{equation}\label{eq:adversarial_reformulation_p}
\max_{v\in [0,1]^k} r_j\|A^\top v\|_q + \sum_{i=1}^{k}\left(1-y_j(w_i^\top x^j + b_i)\right) v_i
\end{equation}
where $\frac{1}{q} + \frac{1}{p}=1$ (respectively $q=1$ if $p=\infty$ and vice versa) and an optimal solution is given by $\delta^*$ with 
\[\delta^* \in  \argmax_{\delta\in U_j} (A^\top v^*)^\top \delta
\]
where $v^*$ is an optimal solution of \eqref{eq:adversarial_reformulation_p}.
\end{theorem}
Note that deriving the optimal solution $\delta^*$ results in optimizing a linear function over the set $U_j$ which can be done efficiently for the $\ell_2$, $\ell_1$ and $\ell_\infty$-norm. However calculating the optimal solution $v^*$ of Problem \eqref{eq:adversarial_reformulation_p} is the main challenge. In the following we derive mixed-integer programming reformulations of the adversarial problem for the $\ell_\infty$ and the $\ell_1$ norm.
\begin{corollary}\label{cor:inf_norm}
If $U_j=\left\{ \delta\in \R^n \ | \ \|\delta\|_\infty \le r_j\right\}$, then the optimal value of the Relaxed Adversarial Problem \eqref{eq:adversarialProblemRelaxed} is equal to the optimal value of
\begin{align}
\max \ & r_j\sum_{l=1}^{n}\nu_l + \sum_{i=1}^{k}\left(1-y_j(w_i^\top x^j + b_i)\right) v_i \label{eq:MILPRelaxedAdversarialInfNorm}\\
s.t. \quad & \nu_l=\left(-1+2\tau_l\right)\left(\sum_{i=1}^{k} (w_i)_l v_i\right) \quad \forall \ l\in [n] \label{constr:abs} \\
& \sum_{i=1}^{k} (w_i)_lv_i \le M_l\tau_l \quad \forall \ l\in [n]\label{constr:bigM1}\\
& \sum_{i=1}^{k} (w_i)_lv_i \ge -M_l(1-\tau_l) \quad \forall \ l\in [n]\label{constr:bigM2}\\
& v\in [0,1]^k, \nu\in \R_+^n, \tau\in \{ 0,1\}^n .
\end{align}
where $M_l=\sum_{i\in [k]} |(w_i)_l|$ for each $l\in [n]$. An optimal solution of Problem \eqref{eq:adversarialProblemRelaxed} is given by $\delta^*$ with 
\[\delta_l^*= \begin{cases}r_j & \text{ if } \sum_{i=1}^{k}-y_j(w_i)_lv_i^*\ge 0 \\ -r_j & \text{ else,} \end{cases}
\]
where $v^*$ is an optimal solution of \eqref{eq:MILPRelaxedAdversarialInfNorm}.
\end{corollary}
\begin{proof}
Applying the results of Theorem \ref{thm:denHertogResult} with $\ell_\infty$-norm we obtain the Problem 
\begin{equation}\label{eq:adversarial_reformulation_inf}
\max_{v\in [0,1]^k} r_j\sum_{l=1}^{n} |\sum_{i=1}^{k}(w_i)_lv_i| + \sum_{i=1}^{k}\left(1-y_j(w_i^\top x^j + b_i)\right) v_i .
\end{equation}
We use variables $\nu_l$ to model the absolute value $|\sum_{i=1}^{k}(w_i)_lv_i|$. To this end note that if $\sum_{i=1}^{k}(w_i)_lv_i > 0$, then $\tau_l=1$ must hold because of Constraint \eqref{constr:bigM1} and Constraint \eqref{constr:bigM2} is not violated by the choice of $M_l$. In this case Constraint \eqref{constr:abs} ensures that $v_l=\sum_{i=1}^{k}(w_i)_lv_i$. On the other hand if $\sum_{i=1}^{k}(w_i)_lv_i < 0$, then $\tau_l=0$ must hold because of Constraint \eqref{constr:bigM2} and Constraint \eqref{constr:bigM1} is not violated by the choice of $M_l$. In this case Constraint \eqref{constr:abs} ensures that $v_l=-\sum_{i=1}^{k}(w_i)_lv_i$. By Theorem \ref{thm:denHertogResult} an optimal solution $\delta^*$ is then given by any
$\delta^* \in \argmax_{\|\delta\|_\infty\le r_j} \left(A^\top v^*\right)^\top \delta$. It is easy to see that the defined $\delta^*$ is an optimal solution of the latter problem.
\end{proof}
Next we derive a similar result as in the latter corollary for the $\ell_1$-norm case.
\begin{corollary}\label{cor:1_norm}
If $U_j=\left\{ \delta\in \R^n \ | \ \|\delta\|_1 \le r_j\right\}$, then the optimal value of the Adversarial Problem \eqref{eq:adversarialProblemRelaxed} is equal to the optimal value of
\begin{align}
\max \ & r_j\sum_{l=1}^{n}\mu_l\nu_l + \sum_{i=1}^{k}\left(1-y_j(w_i^\top x^j + b_i)\right) v_i \label{eq:MILPRelaxedAdversarial1Norm}\\
s.t. \quad & \sum_{l=1}^{n}\mu_l=1 \label{constr:u} \\
& \nu_l + \sum_{i=1}^{k} (w_i)_lv_i \le M_l\tau_l \quad \forall \ l\in [n]\label{constr:bigM3}\\
& \nu_l - \sum_{i=1}^{k} (w_i)_lv_i \le M_l(1-\tau_l) \quad \forall \ l\in [n]\label{constr:bigM4}\\
& v\in [0,1]^k, \mu\in \{0,1\}^n, \tau\in \{ 0,1\}^n
\end{align}
where $M_l=2\sum_{i\in [k]} |(w_i)_l|$ for each $l\in [n]$. An optimal solution is given by $\delta^*$ with $\delta_{l^*}^*=\text{sgn}(\sum_{i=1}^{k}-y_j(w_i)_{l^*}v_i^*)r_j$ for exactly one $l^*\in\argmax_{l\in [n]}|\sum_{i=1}^{k}(w_i)_{l}v_i^*|$ and $\delta_l=0$ otherwise, where $v^*$ is an optimal solution of \eqref{eq:MILPRelaxedAdversarial1Norm}.
\end{corollary}
\begin{proof}
First note that due to the variable bounds it holds $\sum_{i=1}^{k} (w_i)_lv_i\le \frac{1}{2} M_l$ for all $l\in [n]$. From Theorem \eqref{thm:denHertogResult} we obtain that Problem \eqref{eq:adversarialProblemRelaxed} is equivalent to
\begin{equation}\label{eq:adversarial_reformulation_1}
\max_{v\in [0,1]^k} r_j\max_{l\in [n]}|\sum_{i=1}^{k}(w_i)_lv_i| + c^\top v .
\end{equation}
We now verify that in an optimal solution of \eqref{eq:MILPRelaxedAdversarial1Norm} it always holds $\nu_l=|\sum_{i=1}^{k}(w_i)_lv_i|$ for all $l\in [n]$ which shows the first part of the result. First note that in the case $\tau_l=0$ we obtain $\nu_l\le -\sum_{i=1}^{k}(w_i)_lv_i$ by Constraint \eqref{constr:bigM3} while Constraint \eqref{constr:bigM4} is not violated by the choice of $M_l$. On the other hand if $\tau_l=1$ we obtain $\nu_l\le \sum_{i=1}^{k}(w_i)_lv_i$ by Constraint \eqref{constr:bigM4} while Constraint \eqref{constr:bigM3} is not violated by the choice of $M_l$. Since we maximize over $\nu_l$ with positive objective coefficients one of the two constraints will be fulfilled with equality in an optimal solution. Clearly we want to choose the option where $\nu_l$ can be larger, hence it always holds $\nu_l=|\sum_{i=1}^{k}(w_i)_lv_i|$.  The variables $\mu_l$ then choose the maximum value over all $\nu_l$ due to Constraint \eqref{constr:u} and the objective term $r_j\sum_{l=1}^{n}u_l\nu_l$, which exactly models Problem \eqref{eq:adversarial_reformulation_1}.

To obtain an optimal solution we have to solve a linear Problem over the $\ell_1$-ball. An optimal solution is always given by choosing the index where the coefficient has the largest absolute value and increasing/decreasing the corresponding solution variable as much as possible. Then all other variables are set to $0$ which is exactly the solution described in the corollary.
\end{proof}

Note that both in Corollary \ref{cor:inf_norm} and \ref{cor:1_norm} the Relaxed Adversarial Problem is modeled as a mixed integer problem with $\mathcal O(n)$ binary variables. Furthermore as in the Exact Adversarial Problem the Relaxed versions both contain big-M constraints and hence can also be very challenging. Nevertheless since the number of binary variables is linear in $n$, for data sets with a small number of features the relaxed versions can be computationally beneficial. On the other hand if $k$ is a small value the exact version in Section \ref{sec:MRO-SVM} can be better.

\subsection{Heuristic Adversarial Algorithm}\label{sec:heuristicadversarial}
In this section we derive an efficient heuristic to find useful adversarial perturbations, i.e. feasible solutions for Problem \eqref{eq:adversarialProblem}. The main idea behind the method is to calculate the adversarial perturbation as a weighted average of the normal vectors of the given hyperplanes, where the weight for a hyperplane is larger if the distance of the considered data point to the hyperplane is larger. This is motivated by the fact that if a hyperplane is close to the data point, we do not have to go far into this direction to fool this hyperplane. In this case it is more attractive to go into the directions of hyperplanes which are far away. Note that in Algorithm \ref{alg:heuristic} we sort out all hyperplanes which cannot be fooled by an adversarial perturbation, hence the weights for hyperplanes which are too far away are implicitly set to zero.

Given a data point $x^j$ and hyperplanes \[H_1=\{ x\in \R^n: w_1^\top x + b_1\} , \ldots , H_k=\{ x\in \R^n: w_k^\top x + b_k\}\]
we define weights $\beta_1,\ldots \beta_k\ge 0$ where
\begin{equation}\label{eq:betas}
\beta_i:= [1+y^j(w_i^\top x^j + b_i)]_+
\end{equation}
is the weight for hyperplane $H_i$ for each $i\in [k]$. Note that the weights are defined as an adversarial version of the hinge-loss, namely the value $\beta_i$ is large if the data point $x^j$ is on the correct side of the hyperplane and far away. In this case the adversarial is interested in perturbing it to fool the corresponding hyperplane. The closer $x^j$ is to the hyperplane the smaller is the value $\beta_i$ and $\beta_i=1$ if $x^j\in H_i$. For data points which lie already on the wrong side of the hyperplane it holds $\beta_i\in [0,1]$, where $\beta_i>0$ if the point is close to the hyperplane. This is motivated by the fact that the adversarial is not interested in hyperplanes where the point lies already on the wrong side of the hyperplane and is far away from it. If the point is on the wrong side but close to the hyperplane, the normal vector of the hyperplane should get a small weight since perturbing the data point into a completely opposite direction could lead to a correct classification although the hyperplane was already fooled. We normalize the weights $\beta_i$, i.e. we define
\begin{equation}\label{eq:normalized_betas}
\tilde \beta_i:= \frac{\beta_i}{\sum_{i\in [k]} \beta_i}.
\end{equation}

We now define the \textit{heuristic adversarial perturbation} for data point $x^j$ as 
\begin{equation}\label{eq:heuristic_delta}
\delta^j:= r_j\frac{\sum_{i\in [k]} \tilde \beta_i w_i}{\|\sum_{i\in [k]} \tilde \beta_i w_i\|} .
\end{equation}
Note that by the latter normalization it always holds $\| \delta^j\|=r_j$ and therefore $\delta^j\in U_j$. The described heuristic procedure is summarized in Algorithm \ref{alg:heuristic_adversarial}.
\begin{algorithm}\caption{(Heuristic Adversarial Algorithm)}\label{alg:heuristic_adversarial}
\begin{algorithmic}
\State {\bfseries Input:} $n$, $m$, $x^j$, $U_j$, $k$, $H_1,\ldots ,H_k$
	\State {\bfseries Output:} Adversarial perturbation $\delta^j\in U_j$.
	\State Calculate weights $\beta_i$ for each $i\in [k]$ as in \eqref{eq:betas}.
	\State Normalize the weights as in \eqref{eq:normalized_betas}.
	\State Define $\delta^j$ as in \eqref{eq:heuristic_delta}.	
	\State \textbf{Return:} $\delta^j$
\end{algorithmic}
\end{algorithm}

\begin{example}\label{ex:adversarial_heuristic}
Consider the example with two hyperplanes $H_1=\{x\in \R^2: \ -x_1 + x_2 = 0\}$ and $H_2=\{x\in \R^2: \ x_1 + x_2 -2 = 0\}$ and data point $x=(\frac{3}{5},\frac{1}{2})^\top$ with label $y=-1$ (see Figure \ref{fig:exampleAdversarialHeuristic}). Then by the definitions in \eqref{eq:betas} and \eqref{eq:normalized_betas} we obtain $\tilde \beta_1= \frac{11}{30}$ and $\tilde \beta_2=\frac{19}{30}$. Using the definition \eqref{eq:heuristic_delta} we obtain
\[
\delta = \frac{1}{\sqrt{(\frac{4}{15})^2 + 1}}\begin{pmatrix} \frac{4}{15} \\ 1\end{pmatrix}
\] 
and it can easily be verified that $x+\delta$ is classified with label $\hat y=1$ by both hyperplanes.
\end{example}

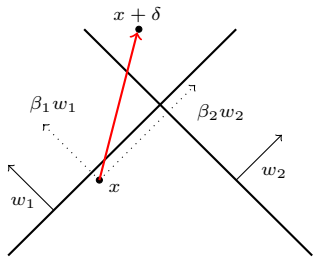
\begin{figure}[htb]
\centering
\begin{tikzpicture}[scale=2]
\draw[thick] (0,0) -- (1.5,1.5);
\draw[thick] (0.5,1.5) -- (2,0);

\node (1) at (0.6,0.5) [draw,fill,circle,inner sep=0.8pt] {};
\node (2) at (0.86,1.5) [draw,fill,circle,inner sep=0.8pt] {};
\node (x) at (0.7,0.45) {\scriptsize $x$};
\node (xd) at (0.85,1.6) {\scriptsize $x+\delta$};
\node (w1) at (0.1,0.35) {\scriptsize $w_1$};
\node (w2) at (1.75,0.55) {\scriptsize $w_2$};
\node (bw1) at (0.3,1) {\scriptsize $\beta_1 w_1$};
\node (bw2) at (1.4,0.95) {\scriptsize $\beta_2 w_2$};
\draw [->] (1.5,0.5) -- (1.8,0.8);
\draw [->] (0.3,0.3) -- (0,0.6);
\draw [->,dotted] (0.6,0.5) -- (0.23,0.87);
\draw [->,dotted] (0.6,0.5) -- (1.23,1.13);
\draw [->, thick, red] (0.6,0.5) -- (0.85,1.48);



\end{tikzpicture}

\caption{Heuristic adversarial perturbation for Example \ref{ex:adversarial_heuristic}.}
\label{fig:exampleAdversarialHeuristic}
\end{figure}

\section{Experiments}\label{sec:computations}
In this section we test the performance of the robust ensemble method given by Algorithm \ref{alg:heuristic} on random and realistic datasets for all adversarial problems derived in Section \ref{sec:MRO-SVM}. All methods are compared to the classical robust SVM and an SVM ensemble method based on bagging. In the following we denote Algorithm \ref{alg:heuristic} with exact adversarial problem \eqref{eq:adversarialProblem} by Ens-E, with relaxed adversarial problem \eqref{eq:adversarialProblemRelaxed} by Ens-R, and with heuristic adversarial perturbation derived by Algorithm \ref{alg:heuristic_adversarial} by Ens-H. The Robust SVM Problem \eqref{eq:ro-svm} is denoted by RO-SVM and the non-robust SVM ensemble method is denoted by SVM-Ens. 

All algorithms were implemented in Python 3.7 using the module scikit-learn 1.0 (\cite{scikit-learn}) and the optimization solver Gurobi 9.1.1 (\cite{gurobi}). All optimization problems were implemented in Gurobi using a timelimit of $7200$ seconds for each Gurobi optimization call. All other parameters are set to their default values. For the classical SVM we use the scikit-learn implementation with linear kernel and for SVM-Ens we use the scikit-learn function BaggingClassifier. The corresponding code is made available online\footnote{https://github.com/JannisKu/EnsembleRobustSVM}.

\subsection{Datasets}
We test all algorithms on the datasets shown in Table \ref{tbl:datasets}. The Breast Cancer Wisconsin dataset (BCW) was taken originally from the UCI Machine Learning Repository \cite{Dua:2019}. 
The digits dataset was loaded by the scikit-learn function load\_digits where each data point is a flattened vector of the original $8\times 8$ pixel matrix which shows handwritten digits from $0$ to $9$. The dataset is converted into a binary classification dataset by assigning label $y=1$ to a pre-defined digit and $y=-1$ to all other digits. We consider the two variants for pre-defined digits $3$ and $7$ denoted by Digits(3) and Digits(7) respectively.

Additionally we generate a random dataset (Gaussian) as follows: we draw a cluster of $100$ random points in dimension $n=5$ from a multivariate gaussian distribution where the mean vector is the all-one vector and the covariance matrix is a diagonal matrix where each entry on the diagonal is $1$. All points are assigned the label $y=1$. Then we draw another cluster of $100$ random points in dimension $n=5$ from a multivariate gaussian distribution where the mean vector is the negative all-one vector and the covariance matrix is the same as above. All the points from the second cluster are assigned the label $y=-1$.  

Table \ref{tbl:datasets} shows the number of data points of each dataset (\# inst.), the number of attributes of each data point (\# attr.), the number of classes (\# class.) and the class distribution (\# class distr.), i.e. the percental fraction of data points for each class. 

\begin{table}[h!]
\begin{center}
\small
\begin{tabular}{l|llll}
Dataset & \# inst. & \# attr. & \# class. & class distr.  \\
\hline
Breast Cancer Wisconsin (BCW) & $699$ & $9$ & $2$ & $65.5\%$/$34.5\%$ \\
Digit Dataset (DD) & $1797$ & $64$ & $10$ & $10\%$/\ldots /$10\%$ \\
Random Gaussian (Gaussian) & $200$ & $5$ & $2$ & $50\%$/$50\%$
\end{tabular}
\end{center}
\caption{Properties of the considered datasets.}
\label{tbl:datasets}
\end{table}

\subsection{Computational Setup}
We test all algorithms on all datasets for several attack and defense levels. To this end we normalize each dataset  by the classical standardize method, i.e. each attribute is transformed by subtracting its mean and  dividing the result by the standard deviation of the original attribute which leads to datasets where each attribute has mean zero and standard deviation one. This choice is motivated by the results in \cite{maggioni2021robustsvm} which show that the uncertainty sets which perform best for RO-SVM approach are the ones where the direction of the axes is given by the standard-deviation-vector of the attributes. We imitate this choice by considering norm-ball uncertainty sets as defined in Section \ref{sec:MRO-SVM} applied to the standardized datasets.

In our tests for a given number of $k=15$ hyperplanes, we consider different defense-levels $r_d\in \{ 0.001, 0.01, 0.05, 0.1, 0.25, 0.5\}$ where the defense-level is the same for each data point, i.e. each data point has the same uncertainty set $U_j:=\left\{ \delta\in \R^n: \ \|\delta\|\le r_d\right\}$. As defense-norm we use the $\ell_2$-norm for all methods except Ens-R and the $\ell_\infty$-norm for all methods except Ens-H. Note that while Ens-R is not applicable for the $\ell_2$-norm, we could also apply Ens-H to the $\ell_\infty$-norm. However for the sake of clarity we omit these calculations. 

For each defense-level we generate $5$ random train-test-splits where the size of the training set is $80\%$ of the original dataset. For each train-test-split we train all methods on the training set where we choose the same number of hyperplanes $k$ for SVM-Ens, Ens-E, Ens-R and Ens-H while we calculate one single hyperplane for RO-SVM. Afterwards we test the performances of all solutions on the test set. To this end we attack each data point in the test set by a worst-case attack vector of length $r_a\in\{ 0.0,0.1,0.2,0.3,0.4,0.5,0.75,1.0,1.25,1.5,1.75,2.0\}$ where the worst-case attack for each data point is calculated by solving the exact adversarial problem \eqref{eq:adversarialProblem} for SVM-Ens, Ens-E, Ens-R and Ens-H, and by solving
\[
\max_{\delta\in U} -y^j((w^*)^\top x^j+b^*)
\]
for RO-SVM, where $w^*,b^*$ is an optimal solution of Problem \eqref{eq:ro-svm}. The attack-norm is always the same as the defense-norm. Note that the worst-case attack depends on the data point $x^j$, its corresponding label $y^j$ and on the calculated solution of the considered model.

For each combination of defense and attack-level $(r_d,r_a)$ we calculate the accuracy of each method on the attacked test set, i.e. the percentage of attacked test-data-points which are classified correctly by each method.  We visualize all results in the following subsections using heatmaps showing the improvement in accuracy compared to SVM-Ens, i.e. we show the difference of the accuracy of the considered robust method and the accuracy of the non-robust SVM ensemble method. Additionally we show a trade-off-curve for SVM-Ens, RO-SVM, Ens-E and Ens-H, where we fix a certain defense-level $r_d$ and show the development of the accuracy over increasing attack-levels. As defense-level for each method we choose the defense-value which has the best average accuracy over all attack-levels. Finally we show the development of the training time in seconds for different protection levels.

Additionally to analyze the behaviour of the ensemble methods for different values of the parameter $k$, we perform another experiment where we fix the defense-level to $r_d=0.1$ and the attack-level to $e_a=1.0$ and instead vary the number of hyperplanes in $k\in \{ 5,10,\ldots ,75\}$ for SVM-Ens, Ens-E and Ens-H. We omit calculations for method Ens-R here since its performance turned out to be not competitive which is mainly due to the fact that we cannot use the $\ell_2$-norm defenses for Ens-R. Furthermore we omit calculations for the Digits dataset due to the increasing computation time. We visualize the results by a line plot showing the average accuracy of each method over $k$. All accuracy and runtime values can be found in Tables \ref{tbl:acc_gaussian_l2} -- \ref{tbl:time_digits7_l2} in the Appendix. 

\subsection{Random Gaussian Dataset}
In this subsection we consider the Random Gaussian Dataset (Gaussian). In Figure \ref{fig:Gaussian_l2} the difference in accuracy of the titled methods and SVM-Ens for $\ell_2$-norm defenses is shown. The results indicate that for all defense levels both methods, RO-SVM and Ens-E, have a similar performance as SVM-Ens for small attacks where Ens-E is sometimes slightly better. For medium sized attacks RO-SVM seems to perform worse than SVM-Ens while Ens-E performs better. For large attacks Ens-E outperforms RO-SVM for all defense-levels except $r_d=0.5$. Ens-H has a worse performance for small attacks while it seems to be the best method for large attacks for all defense-levels. Note that in practice one main difficulty is to choose an appropriate defense-level since comparing the performance in a possible validation step is not well-defined. While some defense levels can have a better performance for small attacks at the same time the performance on large attacks can be much worse than for other defense-levels. Hence having a method which is robust for large attacks on all defense levels is beneficial since choosing the right defense-level is less risky. This is the case here for Ens-E and Ens-H. All accuracy values can be found in Table \ref{tbl:acc_gaussian_l2} in the Appendix showing that Ens-E and Ens-H achieve the best performance for most of the attack-levels.

\begin{figure}[h!]
\centering
\includegraphics[scale=0.28]{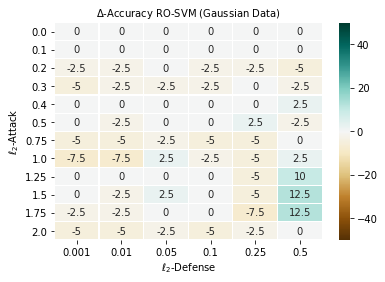}
\includegraphics[scale=0.28]{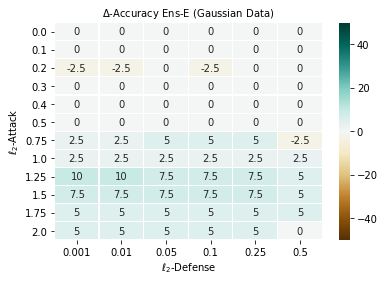}
\includegraphics[scale=0.28]{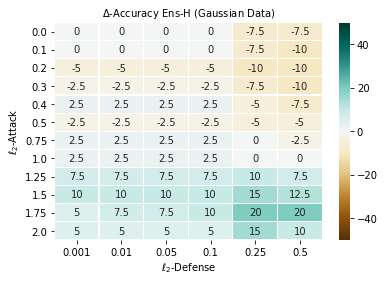}
\caption{Difference between the percentage values of the accuracy of the titled method and the accuracy of SVM-Ens for $\ell_2$-norm defense on the Random Gaussian dataset.}
\label{fig:Gaussian_l2}
\end{figure}

In Figure \ref{fig:Gaussian_linf} the difference in accuracy of the titled methods and the bagging SVM ensemble method (SVM-Ens) for $\ell_\infty$-norm defenses is shown. The results indicate that the performance for this norm is often not better than for the $\ell_2$-norm. While RO-SVM performs well for medium sized attacks it often performs worse for small attacks, except for $r_d=0.5$ defense. The behaviour of Ens-E is not consistent, there are two attack levels ($r_a=0.3$ and $r_a=1.0$) where it outperforms SVM-Ens but has not a real improvement for other attack levels. We find a similar behaviour for Ens-R. All accuracy values can be found in Table \ref{tbl:acc_gaussian_linf} in the Appendix showing that all methods yield the best performance on around a third of the attack-levels. Nevertheless compared to the $\ell_2$-norm the accuracy values are much smaller which is why we omit calculations for the $\ell_\infty$-norm for the subsequent datasets. 

\begin{figure}[h!]
\centering
\includegraphics[scale=0.28]{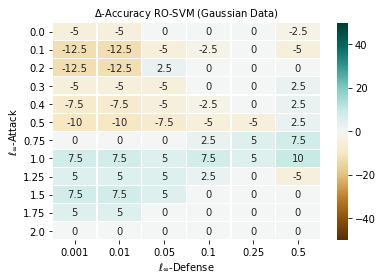}
\includegraphics[scale=0.28]{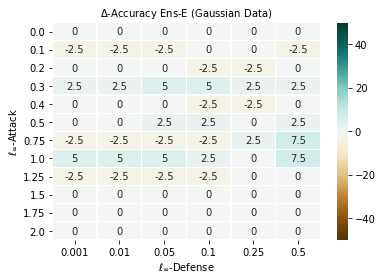}
\includegraphics[scale=0.28]{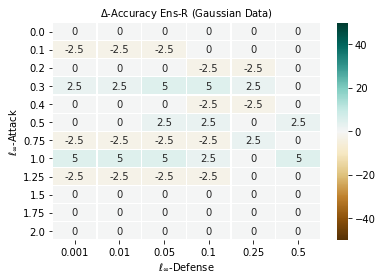}
\caption{Difference between the percentage values of the accuracy of the titled method and the accuracy of SVM-Ens for $\ell_\infty$-norm defense on the Random Gaussian dataset.}
\label{fig:Gaussian_linf}
\end{figure}

In Figure \ref{fig:Gaussian_k} on the left we show the development of the accuracy of the different methods over $k$ for $\ell_2$-norm defense-level $r_d=0.1$ and attack-level $r_a=1.0$. While Ens-E has the best performance for most values of $k$, the performance of Ens-H even decreases for larger $k$. The accuracy of SVM-Ens is constantly in the medium range and always better than RO-SVM. On the right we show the same plot for the $\ell_\infty$-norm. As already observed before the accuracy of all methods is much smaller. Here RO-SVM has the best accuracy while all other methods have a performance varying between $45\%$ and $50\%$. Here for larger $k$ Ens-E and Ens-R seem to be more stable and having its best accuracy. All values can be found in Table \ref{tbl:kDevelop_gaussian_l2} and \ref{tbl:kDevelop_gaussian_linf}.

\begin{figure}[h!]
\centering
\includegraphics[scale=0.4]{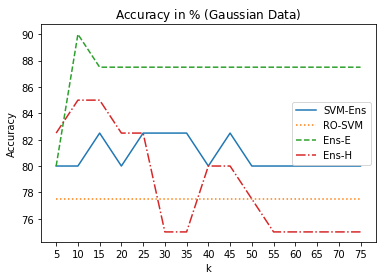}
\includegraphics[scale=0.4]{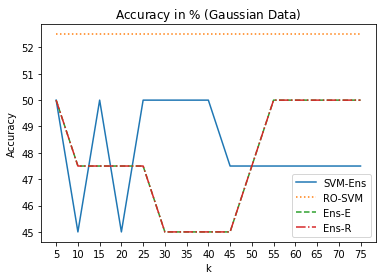}
\caption{Development of the accuracy of the different methods over $k$ for defense-level $r_d=0.1$ and attack-level $r_a=1.0$ for for $\ell_2$-norm (left) and $\ell_\infty$-norm (right).}
\label{fig:Gaussian_k}
\end{figure}

In Figure \ref{fig:Gaussian_tradeoff} we present the trade-off curves of all methods where for each method the defense-level is chosen which has the best average accuracy over all attack-levels. It can be seen that for the $\ell_2$-norm all methods seem to have a better trade-off than SVM-Ens, where Ens-E and Ens-H perform better than RO-SVM for small attacks, while RO-SVM is slightly better for larger attacks. For the $\ell_\infty$-norm the trade-off curves of all method look pretty similar with a slightly better trade-off for Ens-E and Ens-H for some large attacks.

\begin{figure}[h!]
\centering
\includegraphics[scale=0.4]{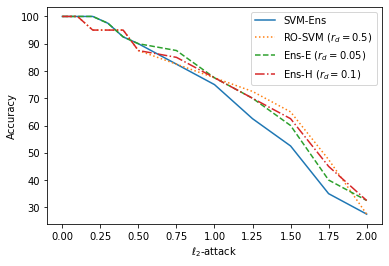}
\includegraphics[scale=0.4]{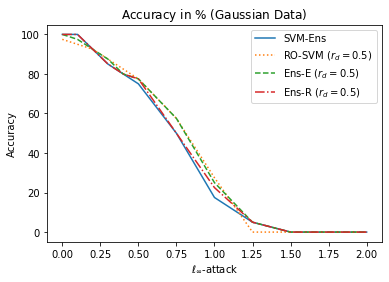}
\caption{Trade-off curves for best-in-average defense-levels for $\ell_2$-norm (left) and $\ell_\infty$-norm (right).}
\label{fig:Gaussian_tradeoff}
\end{figure}

In Figure \ref{fig:Gaussian_runtime} we show the runtime in seconds of the different methods for different defense-levels. SVM-Ens and RO-SVM clearly outperform the robust ensemble methods. Furthermore the runtime of Ens-E increases for larger defense-levels while the runtime of Ens-H seems to have only small increase in runtime. The runtime of Ens-R increases with growing $k$ and is even larger compared to Ens-E. All runtime values can be found in Table \ref{tbl:time_gaussian_l2} and \ref{tbl:time_gaussian_linf} in the Appendix.

\begin{figure}[h!]
\centering
\includegraphics[scale=0.4]{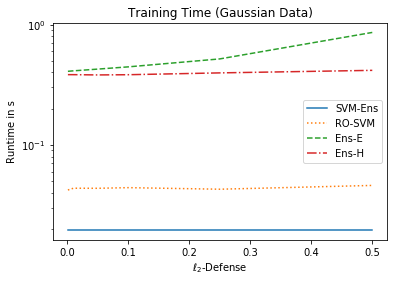}
\includegraphics[scale=0.4]{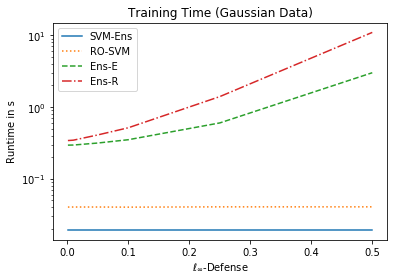}
\caption{Runtime in seconds of the different methods with $\ell_2$-norm (left) and $\ell_\infty$-norm (right) for different defense-levels.}
\label{fig:Gaussian_runtime}
\end{figure}

\subsection{Breast Cancer Wisconsin Dataset}
In this subsection we consider the Breast Cancer Wisconsin dataset (BCW). In Figure \ref{fig:BCW_l2} the difference in accuracy of the titled methods and SVM-Ens for $\ell_2$-norm defenses and attacks is shown. The results indicate that RO-SVM performs similar to SVM-Ens for small to medium attack sizes while it performs better for large attacks, especially for defense-level $r_d=0.5$. On the other hand Ens-E has a slightly worse accuracy for small to medium attacks and a better accuracy for large attacks for most defense-levels. The same effect but even stronger can be observed for Ens-H clearly having the best performance for large attacks for all defense-levels. As for Random Gaussian data the results indicate that the robust ensemble methods Ens-E and Ens-H are much more stable regarding the variation of defense-levels. Nevertheless this comes with a decrease in accuracy for small attacks on the BCW dataset. All accuracy values can be found in Table \ref{tbl:acc_BCW_l2} in the Appendix showing that RO-SVM has the best performance for most of the attack-levels and Ens-H achieves the best performance for the two-largest attack-levels.

\begin{figure}[h!]
\centering
\includegraphics[scale=0.28]{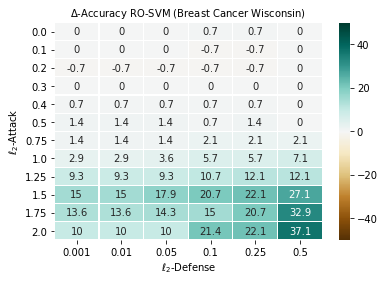}
\includegraphics[scale=0.28]{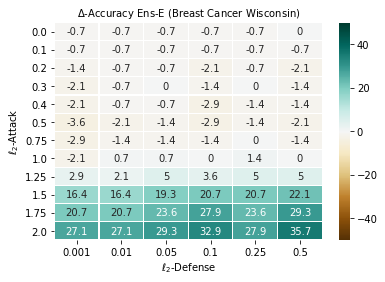}
\includegraphics[scale=0.28]{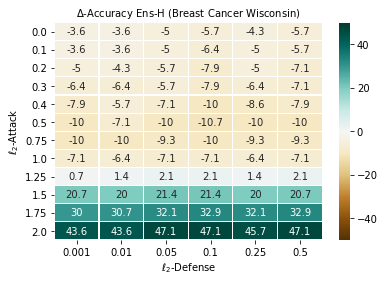}
\caption{Difference between the percentage values of the accuracy of the titled method and the accuracy of SVM-Ens for $\ell_2$-norm defense on BCW.}
\label{fig:BCW_l2}
\end{figure}

In Figure \ref{fig:BCW_k} on the left we show the development of the accuracy of the different methods over $k$ for $\ell_2$-norm defense-level $r_d=0.1$ and attack-level $r_a=1.0$. Here the accuracy of Ens-E and Ens-H decreases drastically for increasing $k$ while the accuracy of SVM-Ens improves. This means for BCW choosing the right $k$ is a challenging task. Nevertheless here the RO-SVM outperforms all other methods. All accuracy values can be found in Table \ref{tbl:kDevelop_BCW_l2}.

On the right in the same figure we present the trade-off curves of all methods where for each method the defense-level is chosen which has the best average accuracy over all attack-levels. It can be seen that RO-SVM has a better overall trade-off curve while Ens-H get better for large attacks.

\begin{figure}[h!]
\centering
\includegraphics[scale=0.4]{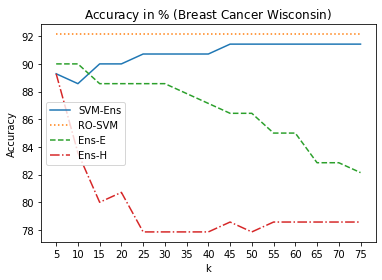}
\includegraphics[scale=0.4]{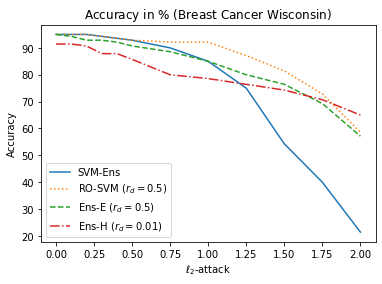}
\caption{Development of the accuracy of the different methods over $k$ for $\ell_2$-norm defense-level $r_d=0.1$ and attack-level $r_a=1.0$. Trade-off curves for best-in-average defense-levels (right).}
\label{fig:BCW_k}
\end{figure}

In Figure \ref{fig:BCW_runtime} we show the runtime in seconds of the different methods for different defense-levels. The computation time of SVM-Ens outperforms the runtime of the other methods. Furthermore the runtime of Ens-E increases for larger defense-levels while the runtime of Ens-H and RO-SVM seems to have only a small increase. All runtime values can be found in Table \ref{tbl:time_BCW_l2} in the Appendix.

\begin{figure}[h!]
\centering
\includegraphics[scale=0.4]{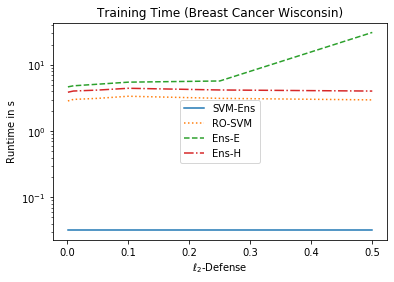}
\caption{Runtime in seconds of the different methods for different defense-levels.}
\label{fig:BCW_runtime}
\end{figure}

\subsection{Digits Dataset}
In this subsection we consider the Digits dataset. As described before we consider two binary classification variants, one were we try to classify digit $3$ (Digits(3)) and one were we try to classify digit $7$ (Digits(7)). 

In Figure \ref{fig:digits3_l2} the difference in accuracy of the titled methods and SVM-Ens for $\ell_2$-norm defenses is shown for Digits(3). The results indicate that RO-SVM performs very bad for small defense levels while it performs very well for $r_d=0.5$ and for large attacks while for small attack-levels there is no significant improvement. Compared to this Ens-E has a very stable accuracy for all defense levels, performing best for large  attacks but never better than RO-SVM for $r_d=0.5$. Note that an advantage here is the stability of the Ens-E method. On the other hand Ens-H clearly outperforms the other methods having a stable behaviour over all defense-levels and even better accuracies than RO-SVM. All accuracy values can be found in Table \ref{tbl:acc_digits3_l2} in the Appendix.

\begin{figure}[h!]
\centering
\includegraphics[scale=0.28]{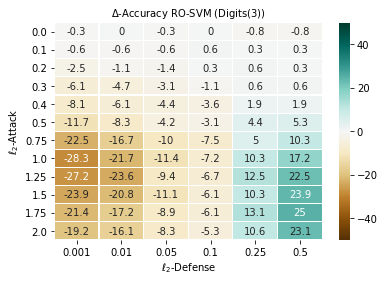}
\includegraphics[scale=0.28]{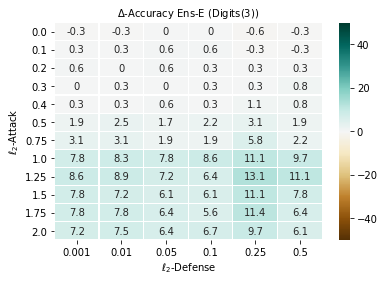}
\includegraphics[scale=0.28]{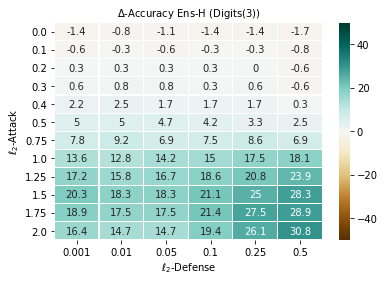}
\caption{Difference between the percentage values of the accuracy of the titled method and the accuracy of SVM-Ens for $\ell_2$-norm defense on Digits(3).}
\label{fig:digits3_l2}
\end{figure}

In Figure \ref{fig:digits7_l2} the difference in accuracy of the titled methods and SVM-Ens for $\ell_2$-norm defenses is shown for Digits(7). The results are pretty similar to the results for Digits(3) but with an increase in accuracy for the Ens-E and Ens-H for large attack-levels. RO-SVM also performs better for $r_d=0.5$ but has significantly worse accuracies for small defense-levels. All accuracy values can be found in Table \ref{tbl:acc_digits7_l2} in the Appendix.

\begin{figure}[h!]
\centering
\includegraphics[scale=0.28]{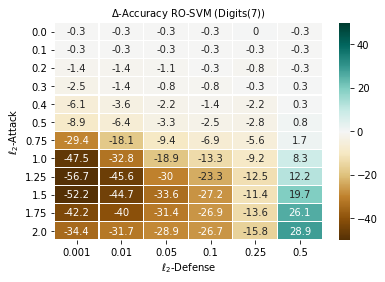}
\includegraphics[scale=0.28]{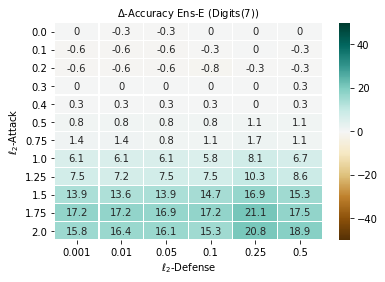}
\includegraphics[scale=0.28]{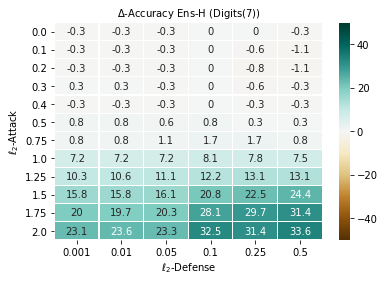}
\caption{Difference between the percentage values of the accuracy of the titled method and the accuracy of SVM-Ens for $\ell_2$-norm defense on Digits(7).}
\label{fig:digits7_l2}
\end{figure}

In Figure \ref{fig:digits_tradeoff} we present the trade-off curves of all methods where for each method the defense-level is chosen which has the best average accuracy over all attack-levels. While for Digits(3) RO-SVM has a slight advantage on medium-sized attacks for larger attacks Ens-R clearly outperforms all other methods for both datasets Digits(3) and Digits(7). Ens-E has a very small advantage for medium-sized attacks on Digits(7) but clearly has a worse trade-off when it comes to larger attacks. All methods clearly outperform SVM-Ens.

\begin{figure}[h!]
\centering
\includegraphics[scale=0.4]{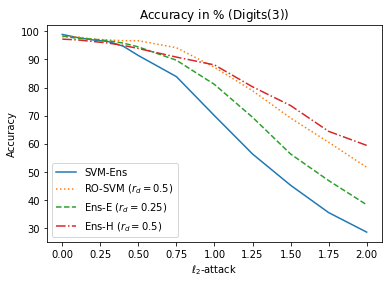}
\includegraphics[scale=0.4]{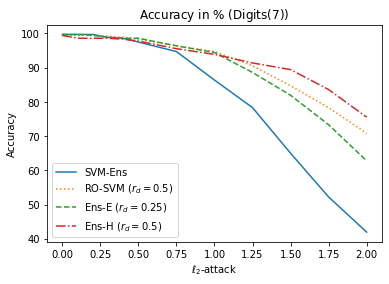}
\caption{Trade-off curves for best-in-average defense-levels for Digits(3) (left) and Digits(7) (right).}
\label{fig:digits_tradeoff}
\end{figure}

In Figure \ref{fig:digits_runtime} we show the runtime in seconds of the different methods for different defense-levels. While RO-SVM and Ens-H have a similar runtime for all defense-levels, the runtime of Ens-E increases for larger defense-levels while the runtime of Ens-H and RO-SVM does not increase much or even decreases for Digits(7). All runtime values can be found in Table \ref{tbl:time_digits3_l2} and \ref{tbl:time_digits7_l2} in the Appendix.

\begin{figure}[h!]
\centering
\includegraphics[scale=0.4]{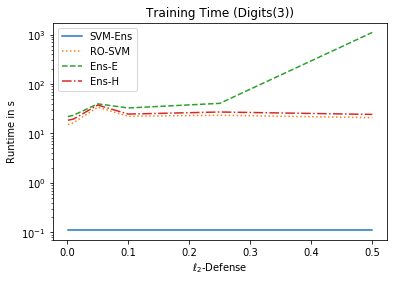}
\includegraphics[scale=0.4]{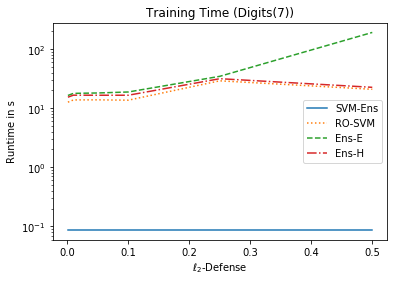}
\caption{Runtime in seconds of the different methods for different defense-levels on Digits(3) (left) and Digits(7) (right).}
\label{fig:digits_runtime}
\end{figure}
 
\section{Conclusion}\label{sec:conclusion}
We present a new iterative ensemble method which tackles data uncertainty and incorporates robustness. The idea is to iteratively solve a classical SVM each time on a perturbed variant of the original dataset where the perturbation is calculated by an adversarial problem. Afterwards a majority vote over all SVM solutions is performed to classify unseen data points. We consider the different variants of the adversarial problem, the exact problem, a relaxed variant and a heuristic variant. All methods are tested on random and realistic datasets and the computations show that the new methods have a much more stable behaviour than the classical robust SVM model when we consider different defense-levels. While on random Gaussian data the new methods slightly outperform the classical robust model and the non-robust SVM ensemble method, on the Digits dataset the heuristic variant is clearly the best model with large improvements in accuracy especially for large attack-levels. On the other hand for the Breast Cancer Wisconsin dataset RO-SVM still performs well while the new models can have better accuracies for large attacks coming along with an accuracy reduction for small attack-levels. Additionally the results show, that $\ell_2$-defenses perform much better than $\ell_\infty$-defenses.

A large drawback of the relaxed variant is that no efficient formulation for $\ell_2$-norm defenses could be derived. Since the computations show a large advantage of the $\ell_2$-norm models this gap should be filled in the future. Furthermore while all solutions which are part of the calculated ensembles are not independent of each other, since in each iteration the adversarial perturbations depend on the previously calculated solutions, future research could consider ensembles where the collaboration of the solutions is improved, i.e. each SVM solution should consider the decisions of all other solutions.   

\small
\bibliography{references}
\bibliographystyle{abbrv}

\newpage
\section*{Appendix}
In the following we show all accuracies and computation times presented in Section \ref{sec:computations}. Bold values are the best values in each row.

\begin{table}[h!]
\tiny
\setlength{\tabcolsep}{1.5pt}
\begin{center}
\begin{tabular}{c|l|llllll|llllll|llllll}
&SVM-Ens & \multicolumn{6}{c|}{RO-SVM} & \multicolumn{6}{c|}{Ens-E} & \multicolumn{6}{c}{Ens-H}\\
\diagbox{$r_a$}{$r_d$} & $0.0$ & $0.001$ & $0.01$ & $0.05$ & $0.1$ & $0.25$ & $0.5$ & $0.001$ & $0.01$ & $0.05$ & $0.1$ & $0.25$ & $0.5$ & $0.001$ & $0.01$ & $0.05$ & $0.1$ & $0.25$ & $0.5$ \\
\hline
0.0&\textbf{100}&\textbf{100}&\textbf{100}&\textbf{100}&\textbf{100}&\textbf{100}&\textbf{100}&\textbf{100}&\textbf{100}&\textbf{100}&\textbf{100}&\textbf{100}&\textbf{100}&\textbf{100}&\textbf{100}&\textbf{100}&\textbf{100}&92.5&92.5\\
0.1&\textbf{100}&\textbf{100}&\textbf{100}&\textbf{100}&\textbf{100}&\textbf{100}&\textbf{100}&\textbf{100}&\textbf{100}&\textbf{100}&\textbf{100}&\textbf{100}&\textbf{100}&\textbf{100}&\textbf{100}&\textbf{100}&\textbf{100}&92.5&90.0\\
0.2&\textbf{100}&97.5&97.5&\textbf{100}&97.5&97.5&95.0&97.5&97.5&\textbf{100}&97.5&\textbf{100}&\textbf{100}&95.0&95.0&95.0&95.0&90.0&90.0\\
0.3&\textbf{97.5}&92.5&95.0&95.0&95.0&\textbf{97.5}&95.0&\textbf{97.5}&\textbf{97.5}&\textbf{97.5}&\textbf{97.5}&\textbf{97.5}&\textbf{97.5}&95.0&95.0&95.0&95.0&90.0&87.5\\
0.4&92.5&92.5&92.5&92.5&92.5&92.5&\textbf{95.0}&92.5&92.5&92.5&92.5&92.5&92.5&\textbf{95.0}&\textbf{95.0}&\textbf{95.0}&\textbf{95.0}&87.5&85.0\\
0.5&90.0&90.0&87.5&90.0&90.0&\textbf{92.5}&87.5&90.0&90.0&90.0&90.0&90.0&90.0&87.5&87.5&87.5&87.5&85.0&85.0\\
0.75&82.5&77.5&77.5&80.0&77.5&77.5&82.5&85.0&85.0&\textbf{87.5}&\textbf{87.5}&\textbf{87.5}&80.0&85.0&85.0&85.0&85.0&82.5&80.0\\
1.0&75.0&67.5&67.5&\textbf{77.5}&72.5&70.0&\textbf{77.5}&\textbf{77.5}&\textbf{77.5}&\textbf{77.5}&\textbf{77.5}&\textbf{77.5}&\textbf{77.5}&\textbf{77.5}&\textbf{77.5}&\textbf{77.5}&\textbf{77.5}&75.0&75.0\\
1.25&62.5&62.5&62.5&62.5&62.5&57.5&\textbf{72.5}&\textbf{72.5}&\textbf{72.5}&70.0&70.0&70.0&67.5&70.0&70.0&70.0&70.0&\textbf{72.5}&70.0\\
1.5&52.5&52.5&50.0&55.0&52.5&47.5&65.0&60.0&60.0&60.0&60.0&60.0&57.5&62.5&62.5&62.5&62.5&\textbf{67.5}&65.0\\
1.75&35.0&32.5&32.5&35.0&35.0&27.5&47.5&40.0&40.0&40.0&40.0&40.0&40.0&40.0&42.5&42.5&45.0&\textbf{55.0}&\textbf{55.0}\\
2.0&27.5&22.5&22.5&25.0&22.5&25.0&27.5&32.5&32.5&32.5&32.5&32.5&27.5&32.5&32.5&32.5&32.5&\textbf{42.5}&37.5
\end{tabular}
\end{center}
\caption{Accuracy (in \%) for the Gaussian dataset with $\ell_2$-norm defense.}
\label{tbl:acc_gaussian_l2}
\end{table}

\begin{table}[h!]
\scriptsize
\setlength{\tabcolsep}{1.5pt}
\begin{center}
\begin{tabular}{c|lllllllllllllll}

$k$ & 5 & 10 & 15 & 20 & 25 & 30 & 35 & 40 & 45 & 50 & 55 & 60 & 65 & 70 & 75 \\
\hline
SVM-Ens&80.0&80.0&\textbf{82.5}&80.0&\textbf{82.5}&\textbf{82.5}&\textbf{82.5}&80.0&\textbf{82.5}&80.0&80.0&80.0&80.0&80.0&80.0\\
RO-SVM&77.5&77.5&77.5&77.5&77.5&77.5&77.5&77.5&77.5&77.5&77.5&77.5&77.5&77.5&77.5\\
Ens-E&80.0&\textbf{90.0}&87.5&87.5&87.5&87.5&87.5&87.5&87.5&87.5&87.5&87.5&87.5&87.5&87.5\\
Ens-H&82.5&\textbf{85.0}&\textbf{85.0}&82.5&82.5&75.0&75.0&80.0&80.0&77.5&75.0&75.0&75.0&75.0&75.0
\end{tabular}
\end{center}
\caption{Accuracy (in \%) for the Gaussian dataset with $\ell_2$-norm defense with $r_d=0.1$ and $\ell_2$-norm attack with $r_a=1.0$.}
\label{tbl:kDevelop_gaussian_l2}
\end{table}

\begin{table}[h!]
\scriptsize
\setlength{\tabcolsep}{1.5pt}
\begin{center}
\begin{tabular}{c|lllllllllllllll}

$k$ & 5 & 10 & 15 & 20 & 25 & 30 & 35 & 40 & 45 & 50 & 55 & 60 & 65 & 70 & 75 \\
\hline
SVM-Ens&\textbf{50.0}&45.0&\textbf{50.0}&45.0&\textbf{50.0}&\textbf{50.0}&\textbf{50.0}&\textbf{50.0}&47.5&47.5&47.5&47.5&47.5&47.5&47.5\\
RO-SVM&52.5&52.5&52.5&52.5&52.5&52.5&52.5&52.5&52.5&52.5&52.5&52.5&52.5&52.5&52.5\\
Ens-E&\textbf{50.0}&47.5&47.5&47.5&47.5&45.0&45.0&45.0&45.0&47.5&\textbf{50.0}&\textbf{50.0}&\textbf{50.0}&\textbf{50.0}&\textbf{50.0}\\
Ens-R&\textbf{50.0}&47.5&47.5&47.5&47.5&45.0&45.0&45.0&45.0&47.5&\textbf{50.0}&\textbf{50.0}&\textbf{50.0}&\textbf{50.0}&\textbf{50.0}
\end{tabular}
\end{center}
\caption{Accuracy (in \%) for the Gaussian dataset with $\ell_\infty$-norm defense with $r_d=0.1$ and $\ell_\infty$-norm attack with $r_a=1.0$.}
\label{tbl:kDevelop_gaussian_linf}
\end{table}

\begin{table}[h!]
\tiny
\setlength{\tabcolsep}{1.5pt}
\begin{center}
\begin{tabular}{c|l|llllll|llllll|llllll}
&SVM-Ens & \multicolumn{6}{c|}{RO-SVM} & \multicolumn{6}{c|}{Ens-E} & \multicolumn{6}{c}{Ens-R}\\
\diagbox{$r_a$}{$r_d$} & $0.0$ & $0.001$ & $0.01$ & $0.05$ & $0.1$ & $0.25$ & $0.5$ & $0.001$ & $0.01$ & $0.05$ & $0.1$ & $0.25$ & $0.5$ & $0.001$ & $0.01$ & $0.05$ & $0.1$ & $0.25$ & $0.5$ \\
\hline
0.0&\textbf{100}&95.0&95.0&\textbf{100}&\textbf{100}&\textbf{100}&97.5&\textbf{100}&\textbf{100}&\textbf{100}&\textbf{100}&\textbf{100}&\textbf{100}&\textbf{100}&\textbf{100}&\textbf{100}&\textbf{100}&\textbf{100}&\textbf{100}\\
0.1&\textbf{100}&87.5&87.5&95.0&97.5&\textbf{100}&95.0&97.5&97.5&97.5&\textbf{100}&\textbf{100}&97.5&97.5&97.5&97.5&\textbf{100}&\textbf{100}&\textbf{100}\\
0.2&92.5&80.0&80.0&\textbf{95.0}&92.5&92.5&92.5&92.5&92.5&92.5&90.0&90.0&92.5&92.5&92.5&92.5&90.0&90.0&92.5\\
0.3&85.0&80.0&80.0&80.0&85.0&85.0&87.5&87.5&87.5&\textbf{90.0}&\textbf{90.0}&87.5&87.5&87.5&87.5&\textbf{90.0}&\textbf{90.0}&87.5&85.0\\
0.4&80.0&72.5&72.5&75.0&77.5&80.0&\textbf{82.5}&80.0&80.0&80.0&77.5&77.5&80.0&80.0&80.0&80.0&77.5&77.5&80.0\\
0.5&75.0&65.0&65.0&67.5&70.0&70.0&\textbf{77.5}&75.0&75.0&\textbf{77.5}&\textbf{77.5}&75.0&\textbf{77.5}&75.0&75.0&\textbf{77.5}&\textbf{77.5}&75.0&\textbf{77.5}\\
0.75&50.0&50.0&50.0&50.0&52.5&55.0&\textbf{57.5}&47.5&47.5&47.5&47.5&52.5&\textbf{57.5}&47.5&47.5&47.5&47.5&52.5&50.0\\
1.0&17.5&25.0&25.0&22.5&25.0&22.5&\textbf{27.5}&22.5&22.5&22.5&20.0&17.5&25.0&22.5&22.5&22.5&20.0&17.5&22.5\\
1.25&5.0&\textbf{10.0}&\textbf{10.0}&\textbf{10.0}&7.5&5.0&0.0&2.5&2.5&2.5&2.5&5.0&5.0&2.5&2.5&2.5&2.5&5.0&5.0\\
1.5&0.0&\textbf{7.5}&\textbf{7.5}&5.0&0.0&0.0&0.0&0.0&0.0&0.0&0.0&0.0&0.0&0.0&0.0&0.0&0.0&0.0&0.0\\
1.75&0.0&\textbf{5.0}&\textbf{5.0}&0.0&0.0&0.0&0.0&0.0&0.0&0.0&0.0&0.0&0.0&0.0&0.0&0.0&0.0&0.0&0.0\\
2.0&\textbf{0.0}&\textbf{0.0}&\textbf{0.0}&\textbf{0.0}&\textbf{0.0}&\textbf{0.0}&\textbf{0.0}&\textbf{0.0}&\textbf{0.0}&\textbf{0.0}&\textbf{0.0}&\textbf{0.0}&\textbf{0.0}&\textbf{0.0}&\textbf{0.0}&\textbf{0.0}&\textbf{0.0}&\textbf{0.0}&\textbf{0.0}
\end{tabular}
\end{center}
\caption{Accuracy (in \%) for the Gaussian dataset with $\ell_\infty$-norm defense.}
\label{tbl:acc_gaussian_linf}
\end{table}

\begin{table}[h!]
\tiny
\setlength{\tabcolsep}{1.5pt}
\begin{center}
\begin{tabular}{c|l|llllll|llllll|llllll}
&SVM-Ens & \multicolumn{6}{c|}{RO-SVM} & \multicolumn{6}{c|}{Ens-E} & \multicolumn{6}{c}{Ens-H}\\
\diagbox{$r_a$}{$r_d$} & $0.0$ & $0.001$ & $0.01$ & $0.05$ & $0.1$ & $0.25$ & $0.5$ & $0.001$ & $0.01$ & $0.05$ & $0.1$ & $0.25$ & $0.5$ & $0.001$ & $0.01$ & $0.05$ & $0.1$ & $0.25$ & $0.5$ \\
\hline
0.0&\textbf{0.0}&\textbf{0.0}&\textbf{0.0}&\textbf{0.0}&\textbf{0.0}&\textbf{0.0}&\textbf{0.0}&0.4&0.4&0.4&0.4&0.5&0.9&0.4&0.4&0.4&0.4&0.4&0.4\\
0.1&\textbf{0.0}&\textbf{0.0}&\textbf{0.0}&\textbf{0.0}&\textbf{0.0}&\textbf{0.0}&\textbf{0.0}&0.4&0.4&0.4&0.4&0.5&0.9&0.4&0.4&0.4&0.4&0.4&0.4\\
0.2&\textbf{0.0}&\textbf{0.0}&\textbf{0.0}&\textbf{0.0}&\textbf{0.0}&\textbf{0.0}&\textbf{0.0}&0.4&0.4&0.4&0.4&0.5&0.9&0.4&0.4&0.4&0.4&0.4&0.4\\
0.3&\textbf{0.0}&\textbf{0.0}&\textbf{0.0}&\textbf{0.0}&\textbf{0.0}&\textbf{0.0}&\textbf{0.0}&0.4&0.4&0.4&0.4&0.5&0.9&0.4&0.4&0.4&0.4&0.4&0.4\\
0.4&\textbf{0.0}&\textbf{0.0}&\textbf{0.0}&\textbf{0.0}&\textbf{0.0}&\textbf{0.0}&\textbf{0.0}&0.4&0.4&0.4&0.4&0.5&0.9&0.4&0.4&0.4&0.4&0.4&0.4\\
0.5&\textbf{0.0}&\textbf{0.0}&\textbf{0.0}&\textbf{0.0}&\textbf{0.0}&\textbf{0.0}&\textbf{0.0}&0.4&0.4&0.4&0.4&0.5&0.9&0.4&0.4&0.4&0.4&0.4&0.4\\
0.75&\textbf{0.0}&\textbf{0.0}&\textbf{0.0}&\textbf{0.0}&\textbf{0.0}&\textbf{0.0}&\textbf{0.0}&0.4&0.4&0.4&0.4&0.5&0.9&0.4&0.4&0.4&0.4&0.4&0.4\\
1.0&\textbf{0.0}&\textbf{0.0}&\textbf{0.0}&\textbf{0.0}&\textbf{0.0}&\textbf{0.0}&\textbf{0.0}&0.4&0.4&0.4&0.4&0.5&0.9&0.4&0.4&0.4&0.4&0.4&0.4\\
1.25&\textbf{0.0}&\textbf{0.0}&\textbf{0.0}&\textbf{0.0}&\textbf{0.0}&\textbf{0.0}&\textbf{0.0}&0.4&0.4&0.4&0.4&0.5&0.9&0.4&0.4&0.4&0.4&0.4&0.4\\
1.5&\textbf{0.0}&\textbf{0.0}&\textbf{0.0}&\textbf{0.0}&\textbf{0.0}&\textbf{0.0}&\textbf{0.0}&0.4&0.4&0.4&0.4&0.5&0.9&0.4&0.4&0.4&0.4&0.4&0.4\\
1.75&\textbf{0.0}&\textbf{0.0}&\textbf{0.0}&\textbf{0.0}&\textbf{0.0}&\textbf{0.0}&\textbf{0.0}&0.4&0.4&0.4&0.4&0.5&0.9&0.4&0.4&0.4&0.4&0.4&0.4\\
2.0&\textbf{0.0}&\textbf{0.0}&\textbf{0.0}&\textbf{0.0}&\textbf{0.0}&\textbf{0.0}&\textbf{0.0}&0.4&0.4&0.4&0.4&0.5&0.9&0.4&0.4&0.4&0.4&0.4&0.4
\end{tabular}
\end{center}
\caption{Training time (in s) for the Gaussian dataset with $\ell_2$-norm defense.}
\label{tbl:time_gaussian_l2}
\end{table}

\begin{table}[h!]
\tiny
\setlength{\tabcolsep}{1.5pt}
\begin{center}
\begin{tabular}{c|l|llllll|llllll|llllll}
&SVM-Ens & \multicolumn{6}{c|}{RO-SVM} & \multicolumn{6}{c|}{Ens-E} & \multicolumn{6}{c}{Ens-R}\\
\diagbox{$r_a$}{$r_d$} & $0.0$ & $0.001$ & $0.01$ & $0.05$ & $0.1$ & $0.25$ & $0.5$ & $0.001$ & $0.01$ & $0.05$ & $0.1$ & $0.25$ & $0.5$ & $0.001$ & $0.01$ & $0.05$ & $0.1$ & $0.25$ & $0.5$ \\
\hline
0.0&\textbf{0.0}&\textbf{0.0}&\textbf{0.0}&\textbf{0.0}&\textbf{0.0}&\textbf{0.0}&\textbf{0.0}&0.3&0.3&0.3&0.4&0.6&3.0&0.3&0.3&0.4&0.5&1.4&10.9\\
0.1&\textbf{0.0}&\textbf{0.0}&\textbf{0.0}&\textbf{0.0}&\textbf{0.0}&\textbf{0.0}&\textbf{0.0}&0.3&0.3&0.3&0.4&0.6&3.0&0.3&0.3&0.4&0.5&1.4&10.9\\
0.2&\textbf{0.0}&\textbf{0.0}&\textbf{0.0}&\textbf{0.0}&\textbf{0.0}&\textbf{0.0}&\textbf{0.0}&0.3&0.3&0.3&0.4&0.6&3.0&0.3&0.3&0.4&0.5&1.4&10.9\\
0.3&\textbf{0.0}&\textbf{0.0}&\textbf{0.0}&\textbf{0.0}&\textbf{0.0}&\textbf{0.0}&\textbf{0.0}&0.3&0.3&0.3&0.4&0.6&3.0&0.3&0.3&0.4&0.5&1.4&10.9\\
0.4&\textbf{0.0}&\textbf{0.0}&\textbf{0.0}&\textbf{0.0}&\textbf{0.0}&\textbf{0.0}&\textbf{0.0}&0.3&0.3&0.3&0.4&0.6&3.0&0.3&0.3&0.4&0.5&1.4&10.9\\
0.5&\textbf{0.0}&\textbf{0.0}&\textbf{0.0}&\textbf{0.0}&\textbf{0.0}&\textbf{0.0}&\textbf{0.0}&0.3&0.3&0.3&0.4&0.6&3.0&0.3&0.3&0.4&0.5&1.4&10.9\\
0.75&\textbf{0.0}&\textbf{0.0}&\textbf{0.0}&\textbf{0.0}&\textbf{0.0}&\textbf{0.0}&\textbf{0.0}&0.3&0.3&0.3&0.4&0.6&3.0&0.3&0.3&0.4&0.5&1.4&10.9\\
1.0&\textbf{0.0}&\textbf{0.0}&\textbf{0.0}&\textbf{0.0}&\textbf{0.0}&\textbf{0.0}&\textbf{0.0}&0.3&0.3&0.3&0.4&0.6&3.0&0.3&0.3&0.4&0.5&1.4&10.9\\
1.25&\textbf{0.0}&\textbf{0.0}&\textbf{0.0}&\textbf{0.0}&\textbf{0.0}&\textbf{0.0}&\textbf{0.0}&0.3&0.3&0.3&0.4&0.6&3.0&0.3&0.3&0.4&0.5&1.4&10.9\\
1.5&\textbf{0.0}&\textbf{0.0}&\textbf{0.0}&\textbf{0.0}&\textbf{0.0}&\textbf{0.0}&\textbf{0.0}&0.3&0.3&0.3&0.4&0.6&3.0&0.3&0.3&0.4&0.5&1.4&10.9\\
1.75&\textbf{0.0}&\textbf{0.0}&\textbf{0.0}&\textbf{0.0}&\textbf{0.0}&\textbf{0.0}&\textbf{0.0}&0.3&0.3&0.3&0.4&0.6&3.0&0.3&0.3&0.4&0.5&1.4&10.9\\
2.0&\textbf{0.0}&\textbf{0.0}&\textbf{0.0}&\textbf{0.0}&\textbf{0.0}&\textbf{0.0}&\textbf{0.0}&0.3&0.3&0.3&0.4&0.6&3.0&0.3&0.3&0.4&0.5&1.4&10.9
\end{tabular}
\end{center}
\caption{Training time (in s) for the Gaussian dataset with $\ell_\infty$-norm defense.}
\label{tbl:time_gaussian_linf}
\end{table}

\begin{table}[h!]
\tiny
\setlength{\tabcolsep}{1.5pt}
\begin{center}
\begin{tabular}{c|l|llllll|llllll|llllll}
&SVM-Ens & \multicolumn{6}{c|}{RO-SVM} & \multicolumn{6}{c|}{Ens-E} & \multicolumn{6}{c}{Ens-H}\\
\diagbox{$r_a$}{$r_d$} & $0.0$ & $0.001$ & $0.01$ & $0.05$ & $0.1$ & $0.25$ & $0.5$ & $0.001$ & $0.01$ & $0.05$ & $0.1$ & $0.25$ & $0.5$ & $0.001$ & $0.01$ & $0.05$ & $0.1$ & $0.25$ & $0.5$ \\
\hline
0.0&95.0&95.0&95.0&95.0&\textbf{95.7}&\textbf{95.7}&95.0&94.3&94.3&94.3&94.3&94.3&95.0&91.4&91.4&90.0&89.3&90.7&89.3\\
0.1&\textbf{95.0}&\textbf{95.0}&\textbf{95.0}&\textbf{95.0}&94.3&94.3&\textbf{95.0}&94.3&94.3&94.3&94.3&94.3&94.3&91.4&91.4&90.0&88.6&90.0&89.3\\
0.2&\textbf{95.0}&94.3&94.3&94.3&94.3&94.3&\textbf{95.0}&93.6&94.3&94.3&92.9&94.3&92.9&90.0&90.7&89.3&87.1&90.0&87.9\\
0.3&\textbf{94.3}&\textbf{94.3}&\textbf{94.3}&\textbf{94.3}&\textbf{94.3}&\textbf{94.3}&\textbf{94.3}&92.1&93.6&\textbf{94.3}&92.9&\textbf{94.3}&92.9&87.9&87.9&88.6&86.4&87.9&87.1\\
0.4&93.6&\textbf{94.3}&\textbf{94.3}&\textbf{94.3}&\textbf{94.3}&\textbf{94.3}&93.6&91.4&92.9&92.9&90.7&92.1&92.1&85.7&87.9&86.4&83.6&85.0&85.7\\
0.5&92.9&\textbf{94.3}&\textbf{94.3}&\textbf{94.3}&93.6&\textbf{94.3}&92.9&89.3&90.7&91.4&90.0&91.4&90.7&82.9&85.7&82.9&82.1&82.9&82.9\\
0.75&90.0&91.4&91.4&91.4&\textbf{92.1}&\textbf{92.1}&\textbf{92.1}&87.1&88.6&88.6&88.6&90.0&88.6&80.0&80.0&80.7&80.0&80.7&80.7\\
1.0&85.0&87.9&87.9&88.6&90.7&90.7&\textbf{92.1}&82.9&85.7&85.7&85.0&86.4&85.0&77.9&78.6&77.9&77.9&78.6&77.9\\
1.25&75.0&84.3&84.3&84.3&85.7&\textbf{87.1}&\textbf{87.1}&77.9&77.1&80.0&78.6&80.0&80.0&75.7&76.4&77.1&77.1&76.4&77.1\\
1.5&54.3&69.3&69.3&72.1&75.0&76.4&\textbf{81.4}&70.7&70.7&73.6&75.0&75.0&76.4&75.0&74.3&75.7&75.7&74.3&75.0\\
1.75&40.0&53.6&53.6&54.3&55.0&60.7&\textbf{72.9}&60.7&60.7&63.6&67.9&63.6&69.3&70.0&70.7&72.1&\textbf{72.9}&72.1&\textbf{72.9}\\
2.0&21.4&31.4&31.4&31.4&42.9&43.6&58.6&48.6&48.6&50.7&54.3&49.3&57.1&65.0&65.0&\textbf{68.6}&\textbf{68.6}&67.1&\textbf{68.6}
\end{tabular}
\end{center}
\caption{Accuracy (in \%) for the BCW dataset with $\ell_2$-norm defense.}
\label{tbl:acc_BCW_l2}
\end{table}

\begin{table}[h!]
\scriptsize
\setlength{\tabcolsep}{1.5pt}
\begin{center}
\begin{tabular}{c|lllllllllllllll}

$k$ & 5 & 10 & 15 & 20 & 25 & 30 & 35 & 40 & 45 & 50 & 55 & 60 & 65 & 70 & 75 \\
\hline
SVM-Ens&89.3&88.6&90.0&90.0&90.7&90.7&90.7&90.7&\textbf{91.4}&\textbf{91.4}&\textbf{91.4}&\textbf{91.4}&\textbf{91.4}&\textbf{91.4}&\textbf{91.4}\\
RO-SVM&92.1&92.1&92.1&92.1&92.1&92.1&92.1&92.1&92.1&92.1&92.1&92.1&92.1&92.1&92.1\\
Ens-E&\textbf{90.0}&\textbf{90.0}&88.6&88.6&88.6&88.6&87.9&87.1&86.4&86.4&85.0&85.0&82.9&82.9&82.1\\
Ens-H&\textbf{89.3}&83.6&80.0&80.7&77.9&77.9&77.9&77.9&78.6&77.9&78.6&78.6&78.6&78.6&78.6
\end{tabular}
\end{center}
\caption{Accuracy (in \%) for the BCW dataset with $\ell_2$-norm defense with $r_d=0.1$ and $\ell_2$-norm attack with $r_a=1.0$.}
\label{tbl:kDevelop_BCW_l2}
\end{table}

\begin{table}[h!]
\tiny
\setlength{\tabcolsep}{1.5pt}
\begin{center}
\begin{tabular}{c|l|llllll|llllll|llllll}
&SVM-Ens & \multicolumn{6}{c|}{RO-SVM} & \multicolumn{6}{c|}{Ens-E} & \multicolumn{6}{c}{Ens-H}\\
\diagbox{$r_a$}{$r_d$} & $0.0$ & $0.001$ & $0.01$ & $0.05$ & $0.1$ & $0.25$ & $0.5$ & $0.001$ & $0.01$ & $0.05$ & $0.1$ & $0.25$ & $0.5$ & $0.001$ & $0.01$ & $0.05$ & $0.1$ & $0.25$ & $0.5$ \\
\hline
0.0&\textbf{0.0}&2.9&3.0&3.1&3.4&3.1&3.0&4.7&4.8&5.1&5.5&5.7&30.8&3.9&4.0&4.2&4.4&4.2&4.0\\
0.1&\textbf{0.0}&2.9&3.0&3.1&3.4&3.1&3.0&4.7&4.8&5.1&5.5&5.7&30.8&3.9&4.0&4.2&4.4&4.2&4.0\\
0.2&\textbf{0.0}&2.9&3.0&3.1&3.4&3.1&3.0&4.7&4.8&5.1&5.5&5.7&30.8&3.9&4.0&4.2&4.4&4.2&4.0\\
0.3&\textbf{0.0}&2.9&3.0&3.1&3.4&3.1&3.0&4.7&4.8&5.1&5.5&5.7&30.8&3.9&4.0&4.2&4.4&4.2&4.0\\
0.4&\textbf{0.0}&2.9&3.0&3.1&3.4&3.1&3.0&4.7&4.8&5.1&5.5&5.7&30.8&3.9&4.0&4.2&4.4&4.2&4.0\\
0.5&\textbf{0.0}&2.9&3.0&3.1&3.4&3.1&3.0&4.7&4.8&5.1&5.5&5.7&30.8&3.9&4.0&4.2&4.4&4.2&4.0\\
0.75&\textbf{0.0}&2.9&3.0&3.1&3.4&3.1&3.0&4.7&4.8&5.1&5.5&5.7&30.8&3.9&4.0&4.2&4.4&4.2&4.0\\
1.0&\textbf{0.0}&2.9&3.0&3.1&3.4&3.1&3.0&4.7&4.8&5.1&5.5&5.7&30.8&3.9&4.0&4.2&4.4&4.2&4.0\\
1.25&\textbf{0.0}&2.9&3.0&3.1&3.4&3.1&3.0&4.7&4.8&5.1&5.5&5.7&30.8&3.9&4.0&4.2&4.4&4.2&4.0\\
1.5&\textbf{0.0}&2.9&3.0&3.1&3.4&3.1&3.0&4.7&4.8&5.1&5.5&5.7&30.8&3.9&4.0&4.2&4.4&4.2&4.0\\
1.75&\textbf{0.0}&2.9&3.0&3.1&3.4&3.1&3.0&4.7&4.8&5.1&5.5&5.7&30.8&3.9&4.0&4.2&4.4&4.2&4.0\\
2.0&\textbf{0.0}&2.9&3.0&3.1&3.4&3.1&3.0&4.7&4.8&5.1&5.5&5.7&30.8&3.9&4.0&4.2&4.4&4.2&4.0
\end{tabular}
\end{center}
\caption{Training time (in s) for the BCW dataset with $\ell_2$-norm defense.}
\label{tbl:time_BCW_l2}
\end{table}

\begin{table}[h!]
\tiny
\setlength{\tabcolsep}{1.5pt}
\begin{center}
\begin{tabular}{c|l|llllll|llllll|llllll}
&SVM-Ens & \multicolumn{6}{c|}{RO-SVM} & \multicolumn{6}{c|}{Ens-E} & \multicolumn{6}{c}{Ens-H}\\
\diagbox{$r_a$}{$r_d$} & $0.0$ & $0.001$ & $0.01$ & $0.05$ & $0.1$ & $0.25$ & $0.5$ & $0.001$ & $0.01$ & $0.05$ & $0.1$ & $0.25$ & $0.5$ & $0.001$ & $0.01$ & $0.05$ & $0.1$ & $0.25$ & $0.5$ \\
\hline
0.0&\textbf{98.9}&98.6&\textbf{98.9}&98.6&\textbf{98.9}&98.1&98.1&98.6&98.6&\textbf{98.9}&\textbf{98.9}&98.3&98.6&97.5&98.1&97.8&97.5&97.5&97.2\\
0.1&97.8&97.2&97.2&97.2&\textbf{98.3}&98.1&98.1&98.1&98.1&\textbf{98.3}&\textbf{98.3}&97.5&97.5&97.2&97.5&97.2&97.5&97.5&96.9\\
0.2&96.9&94.4&95.8&95.6&97.2&\textbf{97.5}&97.2&\textbf{97.5}&96.9&\textbf{97.5}&97.2&97.2&97.2&97.2&97.2&97.2&97.2&96.9&96.4\\
0.3&96.4&90.3&91.7&93.3&95.3&96.9&96.9&96.4&96.7&96.4&96.7&96.7&\textbf{97.2}&96.9&\textbf{97.2}&\textbf{97.2}&96.7&96.9&95.8\\
0.4&94.7&86.7&88.6&90.3&91.1&96.7&96.7&95.0&95.0&95.3&95.0&95.8&95.6&96.9&\textbf{97.2}&96.4&96.4&96.4&95.0\\
0.5&91.4&79.7&83.1&87.2&88.3&95.8&\textbf{96.7}&93.3&93.9&93.1&93.6&94.4&93.3&96.4&96.4&96.1&95.6&94.7&93.9\\
0.75&83.9&61.4&67.2&73.9&76.4&88.9&\textbf{94.2}&86.9&86.9&85.8&85.8&89.7&86.1&91.7&93.1&90.8&91.4&92.5&90.8\\
1.0&70.0&41.7&48.3&58.6&62.8&80.3&87.2&77.8&78.3&77.8&78.6&81.1&79.7&83.6&82.8&84.2&85.0&87.5&\textbf{88.1}\\
1.25&56.4&29.2&32.8&46.9&49.7&68.9&78.9&65.0&65.3&63.6&62.8&69.4&67.5&73.6&72.2&73.1&75.0&77.2&\textbf{80.3}\\
1.5&45.3&21.4&24.4&34.2&39.2&55.6&69.2&53.1&52.5&51.4&51.4&56.4&53.1&65.6&63.6&63.6&66.4&70.3&\textbf{73.6}\\
1.75&35.6&14.2&18.3&26.7&29.4&48.6&60.6&43.3&43.3&41.9&41.1&46.9&41.9&54.4&53.1&53.1&56.9&63.1&\textbf{64.4}\\
2.0&28.6&9.4&12.5&20.3&23.3&39.2&51.7&35.8&36.1&35.0&35.3&38.3&34.7&45.0&43.3&43.3&48.1&54.7&\textbf{59.4}
\end{tabular}
\end{center}
\caption{Accuracy (in \%) for the Digits(3) dataset with $\ell_2$-norm defense.}
\label{tbl:acc_digits3_l2}
\end{table}

\begin{table}[h!]
\tiny
\setlength{\tabcolsep}{1.5pt}
\begin{center}
\begin{tabular}{c|l|llllll|llllll|llllll}
&SVM-Ens & \multicolumn{6}{c|}{RO-SVM} & \multicolumn{6}{c|}{Ens-E} & \multicolumn{6}{c}{Ens-H}\\
\diagbox{$r_a$}{$r_d$} & $0.0$ & $0.001$ & $0.01$ & $0.05$ & $0.1$ & $0.25$ & $0.5$ & $0.001$ & $0.01$ & $0.05$ & $0.1$ & $0.25$ & $0.5$ & $0.001$ & $0.01$ & $0.05$ & $0.1$ & $0.25$ & $0.5$ \\
\hline
0.0&\textbf{0.1}&15.2&16.4&34.0&22.3&23.4&20.9&21.9&23.3&39.7&32.8&40.7&1109.0&18.5&19.5&37.6&24.7&27.2&24.3\\
0.1&\textbf{0.1}&15.2&16.4&34.0&22.3&23.4&20.9&21.9&23.3&39.7&32.8&40.7&1109.0&18.5&19.5&37.6&24.7&27.2&24.3\\
0.2&\textbf{0.1}&15.2&16.4&34.0&22.3&23.4&20.9&21.9&23.3&39.7&32.8&40.7&1109.0&18.5&19.5&37.6&24.7&27.2&24.3\\
0.3&\textbf{0.1}&15.2&16.4&34.0&22.3&23.4&20.9&21.9&23.3&39.7&32.8&40.7&1109.0&18.5&19.5&37.6&24.7&27.2&24.3\\
0.4&\textbf{0.1}&15.2&16.4&34.0&22.3&23.4&20.9&21.9&23.3&39.7&32.8&40.7&1109.0&18.5&19.5&37.6&24.7&27.2&24.3\\
0.5&\textbf{0.1}&15.2&16.4&34.0&22.3&23.4&20.9&21.9&23.3&39.7&32.8&40.7&1109.0&18.5&19.5&37.6&24.7&27.2&24.3\\
0.75&\textbf{0.1}&15.2&16.4&34.0&22.3&23.4&20.9&21.9&23.3&39.7&32.8&40.7&1109.0&18.5&19.5&37.6&24.7&27.2&24.3\\
1.0&\textbf{0.1}&15.2&16.4&34.0&22.3&23.4&20.9&21.9&23.3&39.7&32.8&40.7&1109.0&18.5&19.5&37.6&24.7&27.2&24.3\\
1.25&\textbf{0.1}&15.2&16.4&34.0&22.3&23.4&20.9&21.9&23.3&39.7&32.8&40.7&1109.0&18.5&19.5&37.6&24.7&27.2&24.3\\
1.5&\textbf{0.1}&15.2&16.4&34.0&22.3&23.4&20.9&21.9&23.3&39.7&32.8&40.7&1109.0&18.5&19.5&37.6&24.7&27.2&24.3\\
1.75&\textbf{0.1}&15.2&16.4&34.0&22.3&23.4&20.9&21.9&23.3&39.7&32.8&40.7&1109.0&18.5&19.5&37.6&24.7&27.2&24.3\\
2.0&\textbf{0.1}&15.2&16.4&34.0&22.3&23.4&20.9&21.9&23.3&39.7&32.8&40.7&1109.0&18.5&19.5&37.6&24.7&27.2&24.3
\end{tabular}
\end{center}
\caption{Training time (in s) for the Digits(3) dataset with $\ell_2$-norm defense.}
\label{tbl:time_digits3_l2}
\end{table}

\begin{table}[h!]
\tiny
\setlength{\tabcolsep}{1.5pt}
\begin{center}
\begin{tabular}{c|l|llllll|llllll|llllll}
&SVM-Ens & \multicolumn{6}{c|}{RO-SVM} & \multicolumn{6}{c|}{Ens-E} & \multicolumn{6}{c}{Ens-H}\\
\diagbox{$r_a$}{$r_d$} & $0.0$ & $0.001$ & $0.01$ & $0.05$ & $0.1$ & $0.25$ & $0.5$ & $0.001$ & $0.01$ & $0.05$ & $0.1$ & $0.25$ & $0.5$ & $0.001$ & $0.01$ & $0.05$ & $0.1$ & $0.25$ & $0.5$ \\
\hline
0.0&\textbf{99.7}&99.4&99.4&99.4&99.4&\textbf{99.7}&99.4&\textbf{99.7}&99.4&99.4&\textbf{99.7}&\textbf{99.7}&\textbf{99.7}&99.4&99.4&99.4&\textbf{99.7}&\textbf{99.7}&99.4\\
0.1&\textbf{99.7}&99.4&99.4&99.4&99.4&99.4&99.4&99.2&99.2&99.2&99.4&\textbf{99.7}&99.4&99.4&99.4&99.4&\textbf{99.7}&99.2&98.6\\
0.2&\textbf{99.7}&98.3&98.3&98.6&99.4&98.9&99.4&99.2&99.2&99.2&98.9&99.4&99.4&99.4&99.4&99.4&\textbf{99.7}&98.9&98.6\\
0.3&98.9&96.4&97.5&98.1&98.1&98.6&\textbf{99.2}&98.9&98.9&98.9&98.9&98.9&\textbf{99.2}&\textbf{99.2}&\textbf{99.2}&98.6&98.9&98.3&98.6\\
0.4&98.6&92.5&95.0&96.4&97.2&96.4&\textbf{98.9}&\textbf{98.9}&\textbf{98.9}&\textbf{98.9}&\textbf{98.9}&98.6&\textbf{98.9}&98.3&98.3&98.3&98.6&98.3&98.3\\
0.5&97.5&88.6&91.1&94.2&95.0&94.7&98.3&98.3&98.3&98.3&98.3&\textbf{98.6}&\textbf{98.6}&98.3&98.3&98.1&98.3&97.8&97.8\\
0.75&94.7&65.3&76.7&85.3&87.8&89.2&\textbf{96.4}&96.1&96.1&95.6&95.8&\textbf{96.4}&95.8&95.6&95.6&95.8&\textbf{96.4}&\textbf{96.4}&95.6\\
1.0&86.4&38.9&53.6&67.5&73.1&77.2&\textbf{94.7}&92.5&92.5&92.5&92.2&94.4&93.1&93.6&93.6&93.6&94.4&94.2&93.9\\
1.25&78.3&21.7&32.8&48.3&55.0&65.8&90.6&85.8&85.6&85.8&85.8&88.6&86.9&88.6&88.9&89.4&90.6&\textbf{91.4}&\textbf{91.4}\\
1.5&65.0&12.8&20.3&31.4&37.8&53.6&84.7&78.9&78.6&78.9&79.7&81.9&80.3&80.8&80.8&81.1&85.8&87.5&\textbf{89.4}\\
1.75&52.2&10.0&12.2&20.8&25.3&38.6&78.3&69.4&69.4&69.2&69.4&73.3&69.7&72.2&71.9&72.5&80.3&81.9&\textbf{83.6}\\
2.0&41.9&7.5&10.3&13.1&15.3&26.1&70.8&57.8&58.3&58.1&57.2&62.8&60.8&65.0&65.6&65.3&74.4&73.3&\textbf{75.6}
\end{tabular}
\end{center}
\caption{Accuracy (in \%) for the Digits(7) dataset with $\ell_2$-norm defense.}
\label{tbl:acc_digits7_l2}
\end{table}

\begin{table}[h!]
\tiny
\setlength{\tabcolsep}{1.5pt}
\begin{center}
\begin{tabular}{c|l|llllll|llllll|llllll}
&SVM-Ens & \multicolumn{6}{c|}{RO-SVM} & \multicolumn{6}{c|}{Ens-E} & \multicolumn{6}{c}{Ens-H}\\
\diagbox{$r_a$}{$r_d$} & $0.0$ & $0.001$ & $0.01$ & $0.05$ & $0.1$ & $0.25$ & $0.5$ & $0.001$ & $0.01$ & $0.05$ & $0.1$ & $0.25$ & $0.5$ & $0.001$ & $0.01$ & $0.05$ & $0.1$ & $0.25$ & $0.5$ \\
\hline
0.0&\textbf{0.1}&12.6&13.8&13.8&13.6&29.0&20.8&16.4&17.8&18.1&18.8&34.5&190.7&15.3&16.6&16.5&16.5&31.5&22.5\\
0.1&\textbf{0.1}&12.6&13.8&13.8&13.6&29.0&20.8&16.4&17.8&18.1&18.8&34.5&190.7&15.3&16.6&16.5&16.5&31.5&22.5\\
0.2&\textbf{0.1}&12.6&13.8&13.8&13.6&29.0&20.8&16.4&17.8&18.1&18.8&34.5&190.7&15.3&16.6&16.5&16.5&31.5&22.5\\
0.3&\textbf{0.1}&12.6&13.8&13.8&13.6&29.0&20.8&16.4&17.8&18.1&18.8&34.5&190.7&15.3&16.6&16.5&16.5&31.5&22.5\\
0.4&\textbf{0.1}&12.6&13.8&13.8&13.6&29.0&20.8&16.4&17.8&18.1&18.8&34.5&190.7&15.3&16.6&16.5&16.5&31.5&22.5\\
0.5&\textbf{0.1}&12.6&13.8&13.8&13.6&29.0&20.8&16.4&17.8&18.1&18.8&34.5&190.7&15.3&16.6&16.5&16.5&31.5&22.5\\
0.75&\textbf{0.1}&12.6&13.8&13.8&13.6&29.0&20.8&16.4&17.8&18.1&18.8&34.5&190.7&15.3&16.6&16.5&16.5&31.5&22.5\\
1.0&\textbf{0.1}&12.6&13.8&13.8&13.6&29.0&20.8&16.4&17.8&18.1&18.8&34.5&190.7&15.3&16.6&16.5&16.5&31.5&22.5\\
1.25&\textbf{0.1}&12.6&13.8&13.8&13.6&29.0&20.8&16.4&17.8&18.1&18.8&34.5&190.7&15.3&16.6&16.5&16.5&31.5&22.5\\
1.5&\textbf{0.1}&12.6&13.8&13.8&13.6&29.0&20.8&16.4&17.8&18.1&18.8&34.5&190.7&15.3&16.6&16.5&16.5&31.5&22.5\\
1.75&\textbf{0.1}&12.6&13.8&13.8&13.6&29.0&20.8&16.4&17.8&18.1&18.8&34.5&190.7&15.3&16.6&16.5&16.5&31.5&22.5\\
2.0&\textbf{0.1}&12.6&13.8&13.8&13.6&29.0&20.8&16.4&17.8&18.1&18.8&34.5&190.7&15.3&16.6&16.5&16.5&31.5&22.5
\end{tabular}
\end{center}
\caption{Training time (in s) for the Digits(7) dataset with $\ell_2$-norm defense.}
\label{tbl:time_digits7_l2}
\end{table}

\end{document}